\pdfoutput=1

\documentclass[11pt]{article}

\usepackage[preprint]{acl}

\usepackage{times}
\usepackage{setspace} 
\usepackage{latexsym}
\usepackage{tabularray}  
\usepackage{multirow} 
\usepackage{float}
\usepackage{amssymb}
\usepackage{booktabs}
\usepackage{algorithm}
\usepackage{algpseudocode}
\usepackage{amsmath}
\usepackage{listings}
\usepackage{enumitem}
\usepackage{xcolor}
\usepackage{minted}
\usepackage{colortbl}
\usepackage{fontawesome}
\usepackage[T1]{fontenc}

\definecolor{veronica-red}{RGB}{196,30,58}
\definecolor{ForestGreen}{RGB}{34,139,34}
\definecolor{BrickRed}{rgb}{.72,0,0}
\definecolor{LakeBlue}{RGB}{0,61,153}
\definecolor{lightblue}{RGB}{68,14,196}
\definecolor{lightb}{RGB}{235,245,255}

\newcommand{\fmoon}{\textsuperscript{\fontsize{6pt}{6pt}\selectfont \faMoonO}}

\usepackage[utf8]{inputenc}
\usepackage{microtype}
\usepackage{inconsolata}
\usepackage{subcaption}
\usepackage{graphicx}
\usepackage{makecell}
\newtheorem{theorem}{Theorem}
\newtheorem{proof}{Proof}
\title{ARISE: An Adaptive Resolution-Aware Metric for Test-Time Scaling Evaluation in Large Reasoning Models}

\author{
Zhangyue Yin\textsuperscript{$\diamondsuit$}\quad 
Qiushi Sun\textsuperscript{$\heartsuit$} \quad 
Zhiyuan Zeng\textsuperscript{$\diamondsuit$} \quad 
Zhiyuan Yu\textsuperscript{$\spadesuit$} \\
\bf{
Qipeng Guo\textsuperscript{$\clubsuit$}\fmoon \quad
Xuanjing Huang\textsuperscript{$\diamondsuit$}\textsuperscript{\dag} \quad
Xipeng Qiu\textsuperscript{$\diamondsuit$}\fmoon \textsuperscript{\dag}
}\\
\textsuperscript{$\diamondsuit$}Fudan University \quad
\textsuperscript{$\heartsuit$}The University of Hong Kong \\
\textsuperscript{$\spadesuit$}Nanjing University \quad
\textsuperscript{$\clubsuit$}Shanghai AI Laboratory \quad
\fmoon Shanghai Innovation Institute\\
\texttt{\{yinzy21,cengzy23\}@m.fudan.edu.cn}\quad
\texttt{qiushisun@connect.hku.hk} \\
\texttt{zhiyuan\_yu@smail.nju.edu.cn}\quad
\texttt{guoqipeng@pjlab.org.cn} \\
\texttt{\{xpqiu,xjhuang\}@fudan.edu.cn}
}

\begin{document}
\maketitle

\begin{abstract}
Test-time scaling has emerged as a transformative paradigm for enhancing the performance of large reasoning models, enabling dynamic allocation of computational resources during inference. However, as the landscape of reasoning models rapidly expands, a critical question remains: how can we systematically compare and evaluate the test-time scaling capabilities across different models? In this paper, we introduce ARISE (Adaptive Resolution-aware Scaling Evaluation), a novel metric specifically designed to assess the test-time scaling effectiveness of large reasoning models. Unlike existing evaluation approaches, ARISE incorporates two key innovations: (1) sample-level awareness that effectively penalizes negative scaling behaviors where increased computation leads to performance degradation, 
and (2) a dynamic sampling mechanism that mitigates the impact of accuracy fluctuations and token count instability on the final assessment. We conduct comprehensive experiments evaluating state-of-the-art reasoning models across diverse domains including mathematical reasoning, code generation, and agentic tasks. Our results demonstrate that ARISE provides a reliable and fine-grained measurement of test-time scaling capabilities, revealing significant variations in scaling efficiency across models. Notably, our evaluation identifies Claude Opus as exhibiting superior scaling characteristics compared to other contemporary reasoning models.
\end{abstract}

\section{Introduction}
\label{sec:introduction}

Test-time scaling has emerged as a transformative paradigm in large reasoning models, enabling dynamic computational resource allocation during inference to enhance model performance~\citep{snell2025scaling,wu2025inference}. As an increasing number of models with test-time scaling capabilities are released~\citep{openai2024o1,openai2025gpt5,deepseekai2025deepseekr1,yang2025qwen3}, the ability to scale effectively at inference has become a critical dimension for evaluating model capabilities alongside traditional metrics~\citep{zhang2025survey}. However, systematically comparing test-time scaling effectiveness across diverse models presents significant methodological challenges.

\begin{figure}[t]
  \centering
  \includegraphics[width=0.49\textwidth]{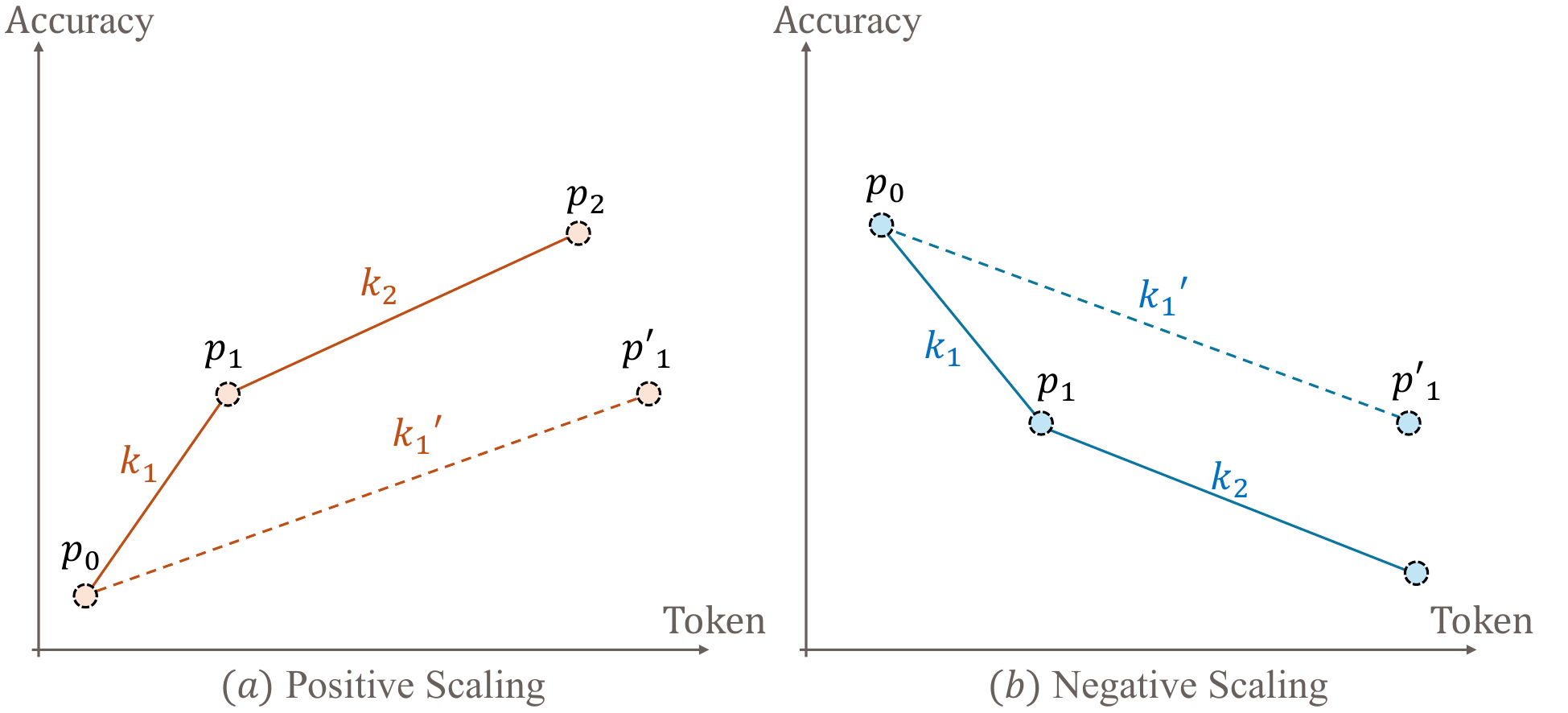}
    \caption{Limitations of slope-based metrics in test-time scaling evaluation. (a) When performance improves from $p_0$ to $p_1$ and $p_1'$, the steeper slope correctly rewards $p_1$ for achieving the same accuracy with fewer tokens. (b) When performance degrades, the slope metric incorrectly assigns a higher value to $p_1'$ despite it wasting more tokens for worse performance.}
  \label{fig:intro}
  \vspace{-1.5em}
\end{figure}

While scaling curves provide intuitive visualization of test-time scaling behavior, 
they lack the quantitative precision necessary for rigorous model comparison. Recent work by \citet{muennighoff2025s1simpletesttimescaling} proposed using slope-based metrics to quantify scaling capabilities. 
However, this approach entails two notable limitations.
First, it operates at the aggregate accuracy level, overlooking a core objective of test-time scaling: \textit{converting previously incorrect samples to correct ones}~\citep{chen2024simple}. This aggregate view fails to penalize samples that become incorrect after scaling. Second, as illustrated in Figure~\ref{fig:intro}, slope metrics exhibit pathological behavior under negative scaling scenarios. When performance improves (Figure~\ref{fig:intro}a), the metric correctly rewards models that achieve higher accuracy with fewer tokens. However, when performance degrades (Figure~\ref{fig:intro}b), the same metric paradoxically assigns higher scores to models that waste more tokens while achieving worse results, a clear misalignment with the intended objective.

Beyond these conceptual limitations, fair measurement of test-time scaling capabilities faces practical challenges. Large reasoning models typically require high sampling temperatures to explore diverse solution paths~\citep{yang2025qwen3}, introducing substantial variability in both token consumption and accuracy measurements. This inherent stochasticity makes it difficult to obtain stable, reproducible assessments of scaling behavior across different models and evaluation runs.

To address these challenges, we introduce ARISE (Adaptive Resolution-aware Scaling Evaluation), a novel metric specifically designed for robust evaluation of test-time scaling in large reasoning models. ARISE incorporates two key innovations: (1) \textbf{sample-level awareness} that tracks individual sample trajectories across scaling levels, effectively penalizing both samples that degrade after scaling and wasteful token consumption under negative scaling; and (2) \textbf{dynamic sampling mechanism} that adaptively adjusts sample-level evaluation runs based on observed variance in accuracy and token consumption, ensuring statistically reliable measurements. These design choices make ARISE a principled and stable metric for assessing test-time scaling capabilities.

We conduct comprehensive experiments evaluating state-of-the-art reasoning models across diverse domains, including mathematical reasoning, code generation, and agentic tasks. 
Our empirical analysis demonstrates that ARISE delivers consistent and fine-grained measurements of test-time scaling effectiveness, exposing substantial disparities in scaling efficiency across models. Notably, Claude Opus outperforms all contemporary models, attaining the highest ARISE scores while exhibiting robust stability across all evaluated task domains. Our contributions are summarized as follows:
\begin{itemize}
\item We introduce scaling efficiency as a critical dimension for evaluating reasoning model abilities and identify fundamental limitations in existing test-time scaling evaluation methods.
\item We propose ARISE, a novel evaluation metric that incorporates sample-level awareness and dynamic sampling mechanism to ensure statistically reliable measurements.
\item We present a comprehensive empirical evaluation across multiple domains, establishing ARISE as a reliable metric for comparing test-time scaling capabilities of reasoning models.
\end{itemize}

\section{Related Work}
\label{sec:related_works}

\paragraph{Test-Time Scaling} 
As training-time scaling approaches its computational and data limits~\citep{villalobos2024position}, test-time scaling has emerged as a promising new paradigm for advancing model capabilities~\citep{zhang2025survey,zeng2025revisiting, wu2025s}. Recent empirical studies have demonstrated that optimal test-time compute allocation can be more effective than simply scaling model parameters~\citep{snell2025scaling,wu2025inference}. The release of OpenAI's o1 model~\citep{openai2024o1} has catalyzed a surge of research into understanding and improving test-time scaling mechanisms~\citep{chen2024simple, hu2025openreasonerzeroopensourceapproach, deepscaler2025}. Major models including OpenAI GPT-5~\citep{openai2025gpt5}, Anthropic Claude~\citep{anthropic2025claude4}, DeepSeek-V3.1~\citep{deepseekai2024deepseekv3} have successfully integrated test-time scaling capabilities. Test-time scaling has been widely applied to a diverse range of complex tasks including mathematical reasoning~\citep{wang2025scaling, balachandran2025inference}, code generation~\citep{yu2025z1, li2025s}, and agentic tasks~\citep{zhu2025scaling, chakraborty2025rolefeedbacktesttimescaling}.

\paragraph{Scaling Evaluation} 
The evaluation of test-time scaling encompasses multiple dimensions that capture different aspects of model behavior. From a performance perspective, Pass@1 remains the most prevalent metric~\citep{yang2025qwen3}, serving as the standard benchmark for mathematical reasoning~\citep{aime,hendrycks2021measuring} and code generation tasks~\citep{raihan2024mhumaneval,jain2025livecodebench}. Extensions such as Pass@k and Cons@k~\citep{chen2021evaluating} provide models with multiple attempts, offering a more comprehensive view of their problem-solving capabilities~\citep{brown2024large,li2022competition}. 
From an efficiency standpoint, recent work has examined computational overhead through various lenses, including token consumption and step redundancy~\citep{luo2025o1,chiang2024reasoning,chen2024not}. 
Notably, \citet{wang2025thoughts} identified a critical phenomenon where reasoning models exhibit excessive trajectory switching during inference, leading to insufficient depth of exploration, a behavior they quantify through the proposed Underthinking Score. 
Another crucial dimension concerns controllability, the ability of models to consistently scale their reasoning process to predetermined computational budgets~\citep{bhargava2023s, muennighoff2025s1simpletesttimescaling,aggarwal2025l1}.

The most comprehensive approach to scaling evaluation employs scaling curves~\citep{wu2025inference,teng2025atom}, which simultaneously capture both accuracy improvements and computational efficiency. These curves visualize the trade-off between performance gains and resource utilization, providing insights into the marginal utility of additional computation. To quantify this relationship, the scaling metric~\citep{muennighoff2025s1simpletesttimescaling} computes the average gradient across all point pairs on the curve:
\begin{equation}
\text{Scaling} = \frac{1}{\binom{|\mathcal{P}|}{2}} \sum_{\substack{p_1, p_2 \in \mathcal{P} \\ \mathcal{T}(p_2) > \mathcal{T}(p_1)}} \frac{\mathcal{A}(p_2) - \mathcal{A}(p_1)}{\mathcal{T}(p_2) - \mathcal{T}(p_1)}
\label{eq:scaling_metric}
\end{equation}
where $\mathcal{P}$ represents the set of points on the scaling curve, and functions $\mathcal{A}(\cdot)$ and $\mathcal{T}(\cdot)$ denote the accuracy and token consumption at each point, respectively. This formulation provides a scalar measure of scaling efficiency but fails to capture sample-level variations and negative scaling behaviors that are crucial for effective evaluation.

\section{ARISE: Adaptive Resolution-Aware Metric}
\label{sec:arise}
We propose ARISE, a novel metric that addresses the limitations of existing test-time scaling evaluation approaches through sample-level error awareness and dynamic sampling mechanisms.

\subsection{Metric Design}
\label{subsec:metric_design}

For each sample $i$ in the evaluation dataset, we define $a_i^{(j)} \in \{0,1\}$ as the binary accuracy at scaling iteration $j$, and $t_i^{(j)}$ as the corresponding token consumption. The ARISE score for sample $i$ is computed as:

\begin{equation}
\text{ARISE}_i = \sum_{j=1}^{m} \Delta a_i^{(j)} \cdot W_i^{(j)}
\label{eq:arise_core}
\end{equation}
where $\Delta a_i^{(j)} = a_i^{(j)} - a_i^{(j-1)}$ represents the accuracy change, and the weight function $W_i^{(j)}$ is defined as:

\begin{equation}
W_i^{(j)} = \left(\frac{t_i^{(j-1)}}{t_i^{(j)}}\right)^{\text{sign}(\Delta a_i^{(j)})}
\label{eq:weight}
\end{equation}
The overall ARISE score aggregates individual sample scores:
\begin{equation}
\text{ARISE} = \frac{1}{n} \sum_{i=1}^{n} \text{ARISE}_i
\label{eq:arise_total}
\end{equation}

\paragraph{Sample-Level Awareness.}
ARISE evaluates scaling behavior at the granularity of individual samples by examining transitions between adjacent scaling iterations. For any pair of consecutive iterations $(j-1, j)$, the contribution to ARISE$_i$ depends on the accuracy transition:

\begin{align}
C_i^{(j)} = \begin{cases}
0 & \text{if } a_i^{(j)} = a_i^{(j-1)} \\
\frac{t_i^{(j-1)}}{t_i^{(j)}} & \text{if } a_i^{(j)} = 1, a_i^{(j-1)} = 0 \\
-\frac{t_i^{(j)}}{t_i^{(j-1)}} & \text{if } a_i^{(j)} = 0, a_i^{(j-1)} = 1
\end{cases}
\label{eq:cases}
\end{align}
This formulation ensures that ARISE captures the critical moment when a sample transitions from incorrect to correct, while penalizing degradations where additional computation leads to errors. Since $t_i^{(j)} > t_i^{(j-1)}$ by construction, the penalty magnitude $|\frac{t_i^{(j)}}{t_i^{(j-1)}}| > |\frac{t_i^{(j-1)}}{t_i^{(j)}}|$ exceeds the reward magnitude, reflecting the asymmetric cost of computational waste.

\paragraph{Negative Scaling Correction.}
When performance deteriorates ($\Delta a_i^{(j)} < 0$), the sign function in Equation~\ref{eq:weight} becomes $-1$, transforming the weight to:
\begin{equation}
W_i^{(j)} = \left(\frac{t_i^{(j-1)}}{t_i^{(j)}}\right)^{-1} = \frac{t_i^{(j)}}{t_i^{(j-1)}} > 1
\label{eq:negative_weight}
\end{equation}
This design amplifies penalties proportionally to token waste, which uses more computational resources for worse results receives progressively stronger penalties, directly addressing the fundamental limitation of existing metrics that fail to adequately penalize negative scaling behaviors.

\paragraph{Magnitude-Aware Design.}
Unlike conventional scaling metrics that employ absolute differences, ARISE utilizes ratios to enable relative scaling measurement adapted to problem difficulty. Consider a fixed token increment $\Delta t = 1000$:

For simple problems where $t_i^{(j-1)} = 1000$, an additional 1000 tokens doubles the computational budget, representing substantial additional reasoning capacity. Conversely, for complex problems where $t_i^{(j-1)} = 10000$, the same increment represents only a 10\% increase, likely insufficient for meaningful additional analysis. This ratio-based approach ensures that scaling effectiveness is measured relative to the baseline computational requirements, providing more accurate assessments across problems of varying complexity.

\paragraph{Non-Combinatorial Computation.}
ARISE employs adjacent-pair computation rather than exhaustive pairwise combinations, avoiding redundancy and computational complexity. Consider a sequence of four scaling iterations with accuracy pattern $(0, 1, 0, 1)$ and strictly increasing tokens $t_i^{(0)} < t_i^{(1)} < t_i^{(2)} < t_i^{(3)}$. 

The adjacent-pair approach yields:
\begin{equation}
\text{ARISE}_i = \frac{t_i^{(0)}}{t_i^{(1)}} - \frac{t_i^{(2)}}{t_i^{(1)}} + \frac{t_i^{(2)}}{t_i^{(3)}}
\label{eq:adjacent}
\end{equation}
whereas combinatorial computation would include an additional term $\frac{t_i^{(0)}}{t_i^{(3)}}$, inappropriately rewarding the direct transition from initial failure to final success while ignoring the intermediate regression. This spurious reward could exceed the penalty for the intermediate failure, demonstrating why adjacent-pair computation provides more reasonable scaling assessment.

\paragraph{Boundedness Properties.}
Unlike traditional scaling metrics with symmetric bounds, ARISE exhibits asymmetric bounds reflecting its design philosophy. When $a_i^{(j-1)} = 0$ and $a_i^{(j)} = 1$, as $t_i^{(j)} \to t_i^{(j-1)}$, $\text{ARISE}_i \to 1^{-}$. For degradation cases where $a_i^{(j-1)} = 1$ and $a_i^{(j)} = 0$:
\begin{equation}
C_i^{(j)} = -\frac{t_i^{(j)}}{t_i^{(j-1)}} < -1
\label{eq:lower_bound}
\end{equation}
Thus, ARISE $\in (-\infty, 1)$, though practical values typically remain within $(-1, 1)$ as extreme negative scaling is rare. The unbounded negative range ensures severe penalties for egregious computational waste, while the bounded positive range prevents over-rewarding improvements. We provide formal boundedness analysis in Appendix~\ref{app:boundedness}.

\subsection{Adaptive Sampling Strategy}
\label{subsec:adaptive_sampling}

Test-time scaling evaluation faces inherent variance in both accuracy outcomes and token consumption across trials. We introduce an adaptive sampling mechanism that dynamically allocates computational budget based on observed variance patterns, enhancing evaluation reliability.

\paragraph{Variance Characterization.}
For each sample $i$ at scaling iteration $j$, we conduct an initial probing phase with $m_{\text{min}}$ trials. Let $a_{i,k}^{(j)}$ and $t_{i,k}^{(j)}$ denote the accuracy and token consumption for trial $k$. We compute the empirical statistics:

\begin{align}
\mu_{a_i^{(j)}} &= \frac{1}{m_{\text{min}}} \sum_{k=1}^{m_{\text{min}}} a_{i,k}^{(j)} \\
\sigma_{a_i^{(j)}} &= \sqrt{\frac{1}{m_{\text{min}}} \sum_{k=1}^{m_{\text{min}}} (a_{i,k}^{(j)} - \mu_{a_i^{(j)}})^2}
\label{eq:stats}
\end{align}
with analogous definitions for token statistics $\mu_{t_i^{(j)}}$ and $\sigma_{t_i^{(j)}}$.

\paragraph{Normalized Variance Measure.}
To enable fair comparison across different scales and magnitudes, we employ the coefficient of variation (CV):

\begin{align}
\text{CV}_{a_i^{(j)}} &= \frac{\sigma_{a_i^{(j)}}}{\mu_{a_i^{(j)}} + \epsilon} \\
\text{CV}_{t_i^{(j)}} &= \frac{\sigma_{t_i^{(j)}}}{\mu_{t_i^{(j)}} + \epsilon}
\label{eq:cv}
\end{align}
where $\epsilon = 10^{-8}$ prevents division by zero. The combined variance indicator captures both dimensions:
\begin{equation}
\text{CV}_i^{(j)} = \text{CV}_{a_i^{(j)}} + \text{CV}_{t_i^{(j)}}
\label{eq:combined_cv}
\end{equation}

\begin{table*}[htbp]
\centering
\footnotesize
\begin{tabular}{l@{\hspace{0.3em}}rr@{\hspace{0.5em}}rr@{\hspace{0.5em}}rr@{\hspace{0.5em}}rr}
\toprule
\multirow{2}{*}{\textbf{Model}} & \multicolumn{2}{c}{\textbf{AIME}} & \multicolumn{2}{c}{\textbf{HMMT}} & \multicolumn{2}{c}{\textbf{GPQA Diamond}} & \multicolumn{2}{c}{\textbf{MMLU-Pro}} \\
\cmidrule(lr){2-3} \cmidrule(lr){4-5} \cmidrule(lr){6-7} \cmidrule(lr){8-9}
 & ARISE & SM$\times$1000 & ARISE & SM$\times$1000 & ARISE & SM$\times$1000 & ARISE & SM$\times$1000 \\
\midrule
o1                      & 0.1346 & 0.0607 & 0.1277 & 0.0183 & 0.1228 & 0.0385 & 0.1509 & 0.0504 \\
o3                      & 0.2993 & 0.0782 & 0.1673 & 0.0186 & 0.2124 & 0.0589 & 0.2789 & 0.0567 \\
o3-mini                 & 0.1306 & 0.0374 & 0.1888 & 0.0302 & 0.1649 & 0.0221 & 0.1663 & 0.0416 \\
o4-mini                 & 0.2402 & 0.0348 & 0.1673 & 0.0435 & 0.1994 & 0.0423 & 0.2080 & 0.0392 \\
\midrule
gpt-oss-20B            & -0.4030 & 0.0205 & -0.3126 & 0.0168 & -0.3274 & 0.0224 & -0.2694 & 0.0227 \\
gpt-oss-120B           & -0.3340 & 0.0273 & -0.1999 & 0.0241 & -0.2734 & 0.0286 & -0.1615 & 0.0314 \\
gpt-5                  & 0.1566 & 0.0259 & 0.2996 & 0.0265 & 0.2185 & 0.0263 & 0.3186 & 0.0298 \\
\midrule
Claude Sonnet 4        & 0.1041 & 0.0461 & 0.0404 & 0.0109 & 0.0636 & 0.0326 & 0.1059 & 0.0264 \\
Claude Opus 4          & 0.3475 & 0.0653 & 0.1717 & 0.0617 & 0.2212 & 0.0488 & 0.3330 & 0.0675 \\
Claude Opus 4.1        & \textbf{0.4529} & \textbf{0.1462} & \textbf{0.4709} & \textbf{0.1419} & \textbf{0.4454} & \textbf{0.1416} & \textbf{0.4932} & \textbf{0.2038} \\
\midrule
Qwen-3-0.6B            & 0.2936 & 0.0023 & 0.1769 & 0.0071 & 0.2114 & 0.0052 & 0.2716 & 0.0034 \\
Qwen-3-1.7B            & 0.3658 & 0.0278 & 0.2746 & 0.0256 & 0.3237 & 0.0321 & 0.3917 & 0.0345 \\
Qwen-3-4B              & 0.2166 & 0.0496 & 0.2217 & 0.0378 & 0.2400 & 0.0325 & 0.2569 & 0.0619 \\
Qwen-3-8B              & 0.3085 & 0.0342 & 0.3010 & 0.0268 & 0.3022 & 0.0354 & 0.3274 & 0.0357 \\
Qwen-3-14B             & 0.3247 & 0.0684 & 0.2213 & 0.0473 & 0.2594 & 0.0449 & 0.3120 & 0.0801 \\
Qwen-3-32B             & 0.3883 & 0.0767 & 0.2038 & 0.0585 & 0.2644 & 0.0576 & 0.3992 & 0.0658 \\
Qwen3-30B-A3B          & 0.3293 & 0.0503 & 0.3855 & 0.0742 & 0.3727 & 0.0437 & 0.4162 & 0.0691 \\
Qwen3-235B-A22B        & 0.3915 & 0.0819 & 0.4306 & 0.0415 & 0.4069 & 0.0762 & 0.4533 & 0.0794 \\
\midrule
Deepseek-R1            & -0.0318 & 0.0072 & -0.0455 & 0.0046 & -0.0493 & 0.0033 & -0.0108 & 0.0028 \\
V3.1          & 0.3966 & 0.0351 & 0.2048 & 0.0369 & 0.2714 & 0.0427 & 0.3559 & 0.0362 \\
V3.1-Terminus & 0.3241 & 0.0237 & 0.2835 & 0.0353 & 0.3091 & 0.0316 & 0.3228 & 0.0319 \\
V3.2-Exp     & 0.3029 & 0.0276 & 0.2651 & 0.0331 & 0.2726 & 0.0209 & 0.3219 & 0.0293 \\
\bottomrule
\end{tabular}
\vspace{-.5em}
\caption{Performance of mainstream models in mathematical and scientific reasoning. Each benchmark shows ARISE scores and corresponding Scaling Metrics (SM). For improved readability, SM values have been multiplied by 1000. The original unscaled values can be found in Appendix Table~\ref{tab:arise_scaling}.}
\label{tab:arise_scaling_processed}
\vspace{-1.5em}
\end{table*}

\paragraph{Dynamic Sampling Protocol.}
We implement an adaptive sampling strategy with a maximum sampling budget $m_{\max}$ and a convergence threshold $\tau$. For each configuration $(i,j)$, we iteratively collect samples while monitoring the combined coefficient of variation. Sampling continues until either convergence is achieved or the budget is exhausted:
\begin{equation}
\text{CV}_i^{(j)} < \tau \quad \text{or} \quad k = m_{\max}
\label{eq:stopping}
\end{equation}
Upon termination at iteration $k^*$, we compute the final statistics as the empirical means over all collected samples:
\begin{align}
a_i^{(j)} &= \frac{1}{k^*} \sum_{k=1}^{k^*} a_{i,k}^{(j)} \\
t_i^{(j)} &= \frac{1}{k^*} \sum_{k=1}^{k^*} t_{i,k}^{(j)}
\label{eq:final_stats}
\end{align}
This adaptive approach balances statistical reliability with computational efficiency. High-variance configurations receive additional sampling to reduce uncertainty, while stable configurations terminate early to conserve resources. We provide the algorithmic implementation in Appendix~\ref{app:algorithm}.

\paragraph{Budget Allocation.}
For comparative analysis with fixed total budget $B$ across $n$ samples and $J$ iterations, we allocate additional trials proportionally to observed variance:

\begin{equation}
m_i^{(j)} = m_{\min} + \left\lfloor \frac{(B - nJm_{\min}) \cdot \text{CV}_i^{(j)}}{\sum_{i',j'} \text{CV}_{i'}^{(j')}} \right\rfloor
\label{eq:budget_algo}
\end{equation}
This variance-proportional allocation distributes the remaining sampling budget according to the variability observed in the initial probing phase. Cases with greater fluctuations receive more sampling, ensuring a more robust and reliable assessment of scaling across diverse model behaviors.

\section{Experiments}
\label{sec:experiments}

We conduct comprehensive experiments to evaluate ARISE across diverse reasoning tasks and model architectures. Our evaluation encompasses both text-based and multimodal benchmarks, covering mathematical reasoning, scientific problem-solving, code generation, and agentic capabilities.

\begin{figure*}[thp]
    \centering
    \includegraphics[width=\textwidth]{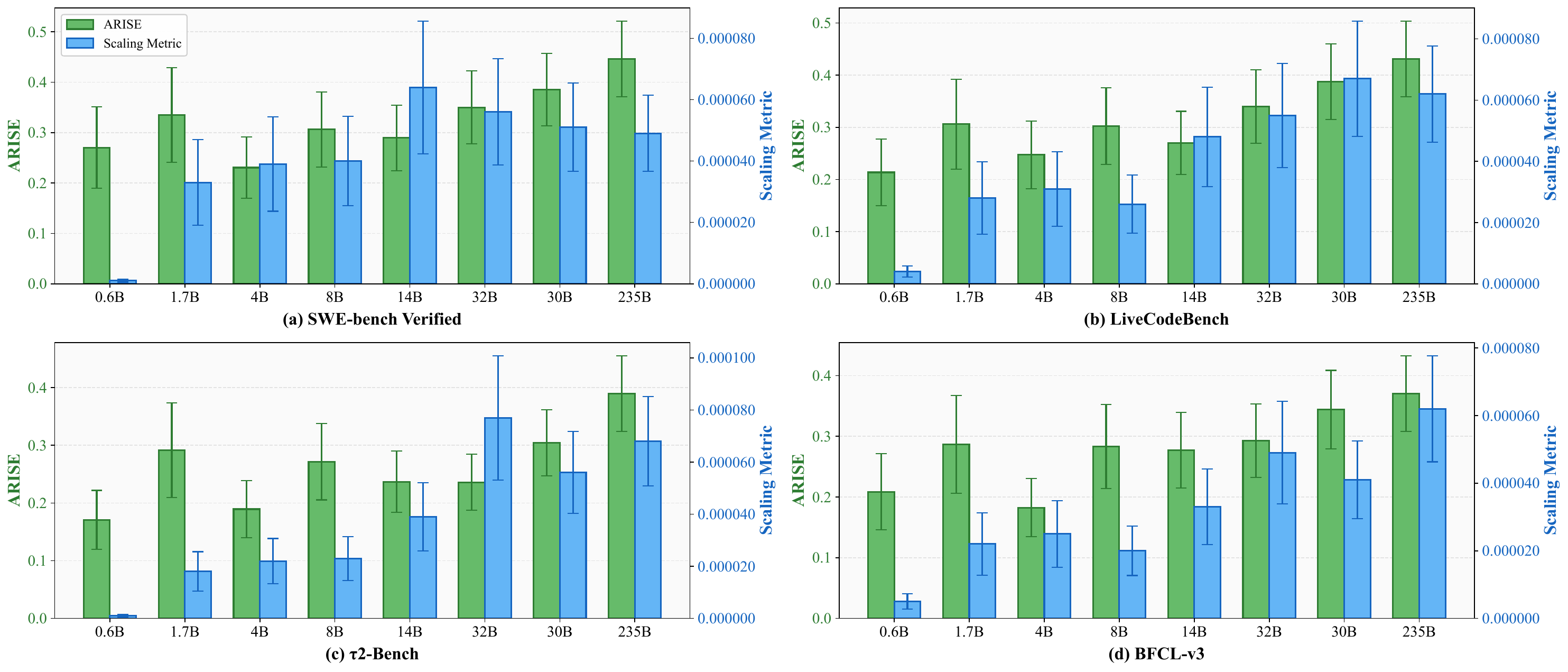}
    \caption{Comparison of ARISE and Scaling Metric across different Qwen3 models on code and agentic tasks. The x-axis shows model parameter counts where 0.6B, 1.7B, 4B, 8B, 14B, 32B correspond to Qwen-3 models, 30B corresponds to Qwen3-30B-A3B, and 235B corresponds to Qwen3-235B-A22B. ARISE values are shown on the left y-axis (green bars) while Scaling Metric values are shown on the right y-axis (blue bars). Error bars represent standard deviations across five independent runs. Complete results for all models are presented in Appendix Table~\ref{tab:arise_scaling_code_agentic}.}
    \label{fig:arise_scaling_comparison}
    \vspace{-1em}
\end{figure*}

\subsection{Experimental Settings}
\label{subsec:experiment_settings}

\paragraph{Evaluation Datasets.} To thoroughly assess the effectiveness of ARISE, we incorporate a comprehensive suite of text-based and multimodal benchmarks. Our evaluation spans four primary categories of reasoning tasks:

\begin{itemize}[itemsep=5pt,topsep=3pt,parsep=0pt]
\item \textbf{Mathematical Reasoning:} AIME~\citep{aime} and HMMT~\citep{balunovic2025srimatharena}
\item \textbf{Scientific Reasoning:} GPQA-diamond~\citep{rein2024gpqa} and MMLU-Pro~\citep{wang2024mmlu}
\item \textbf{Code Generation:} SWE-bench Verified~\citep{Jimenez2024swe} and LiveCodeBench~\citep{jain2025livecodebench}
\item \textbf{Agentic Tasks:} $\tau^2$-Bench~\citep{barres2025tau2} and BFCL-v3~\citep{patil2025bfcl}
\end{itemize}
Additionally, we evaluate multimodal reasoning capabilities on vision-language tasks, including MMMU~\citep{yue2024mmmu}, MathVista~\citep{lu2024mathvista}, and CharXiv-Reasoning~\citep{wang2024charxiv}. These benchmarks test the models' ability to integrate visual and textual information for complex reasoning. Detailed descriptions of each dataset, including sample counts and evaluation metrics, are provided in Appendix~\ref{app:dataset_details}.

\vspace{-.5em}
\paragraph{Implementation Details.}
We evaluate both proprietary and open-source reasoning models through their respective interfaces. For the main experiments, we configure the adaptive sampling parameters with $m_{\min} = 3$, $m_{\max} = 10$, and convergence threshold $\tau = 0.5$. To assess metric stability (Section~\ref{subsec:analysis}), we conduct five independent runs and report standard deviations. For the adaptive sampling analysis under fixed budget constraints (Section~\ref{subsec:adaptive_sampling}), we set the total budget $B = 5nJ$, where $n$ denotes the number of samples and $J$ represents the number of scaling iterations. Additional implementation details are documented in Appendix~\ref{app:experiment_details}.

\subsection{Main Results}
\label{subsec:main_results}

\begin{figure*}[thp]
    \centering
    \includegraphics[width=\textwidth]{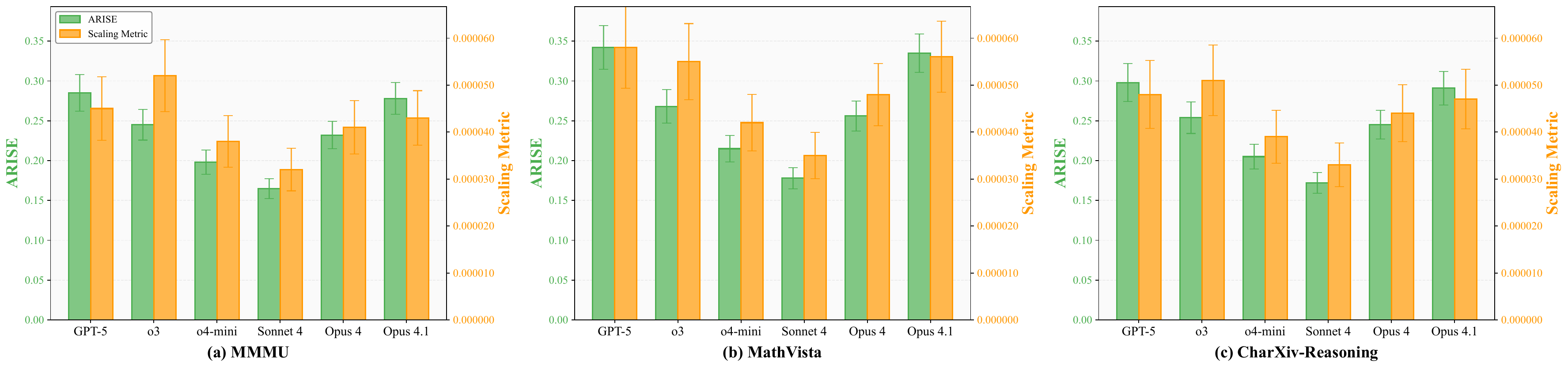}
    \caption{Comparison of ARISE and Scaling Metric across state-of-the-art reasoning models on multimodal reasoning tasks. The evaluation encompasses three challenging benchmarks: (a) MMMU, (b) MathVista, and (c) CharXiv-Reasoning. ARISE values are shown on the left y-axis (green bars) while Scaling Metric values are shown on the right y-axis (orange bars).}
    \label{fig:multimodal_arise_scaling_estimation}
    \vspace{-1em}
\end{figure*}

\paragraph{Mathematical and Scientific Reasoning.} 
Table~\ref{tab:arise_scaling_processed} presents our evaluation of models on mathematical and scientific reasoning benchmarks. Claude Opus 4.1 consistently achieves the highest performance on both ARISE and traditional Scaling Metric, with ARISE scores consistently exceeding 0.45 across most benchmarks and reaching 0.493 on MMLU-Pro. Among open-source models, Qwen3-235B-A22B demonstrates the strongest test-time scaling capabilities, achieving ARISE scores above 0.39 on all evaluated tasks, approaching commercial models. Notably, the ARISE metric reveals nuanced scaling behaviors that traditional metrics miss, particularly the negative scaling phenomena exhibited by certain models where increased computation degrades performance.

\vspace{-.5em}
\paragraph{Code Generation and Agentic Tasks.} 
Figure~\ref{fig:arise_scaling_comparison} illustrates the comparative analysis of ARISE and Scaling Metric on code generation and agentic tasks across the Qwen3 model family. We observe a consistent positive correlation between model parameter size and both metrics, validating that larger models generally exhibit superior test-time scaling capabilities. Interestingly, code generation tasks consistently achieve higher ARISE scores compared to agentic tasks across all model sizes, suggesting that programming problems benefit more substantially from test-time scaling than tool-use.

\vspace{-.5em}
\paragraph{Multimodal Reasoning.} 
Figure~\ref{fig:multimodal_arise_scaling_estimation} presents our analysis on multimodal reasoning tasks. GPT-5 and Claude Opus 4.1 achieve the highest performance across all three benchmarks, with GPT-5 demonstrating a marginal advantage on CharXiv-Reasoning, particularly for complex scientific chart interpretation. Compared to the Scaling Metric, ARISE provides significantly better discriminative power. The performance gap between weaker models (e.g., Sonnet 4) and stronger models (e.g., Opus 4.1) becomes more pronounced under ARISE evaluation, with differences exceeding 75.2\% versus 45.6\% in Scaling Metrics.

\subsection{Analysis}
\label{subsec:analysis}

\paragraph{Penalization of Performance Degradation.} 
Table~\ref{tab:arise_scaling_processed} reveals a critical distinction between ARISE and Scaling Metric in handling negative scaling behaviors. While both metrics generally provide concordant evaluations, ARISE uniquely captures performance degradation through negative scores. For instance, GPT-OSS-20B achieves -0.403 on AIME and DeepSeek-R1 scores -0.049 on GPQA Diamond, indicating that many initially correct samples become incorrect with increased computation, a phenomenon consistent with recent findings by \citet{zeng2025revisiting}. Scaling Metric remains insensitive to this crucial behavior, assigning only positive values regardless of degradation, failing to effectively penalize computational waste. We provide a detailed analysis of sample accuracy transitions before and after scaling in Appendix~\ref{app:analysis}.

\paragraph{Evaluation Stability.} 
The error bars in Figure~\ref{fig:arise_scaling_comparison} demonstrate ARISE's superior stability compared to traditional metrics. When model performances are obvious (e.g., 32B, 30B, and 235B variants), ARISE exhibits notably smaller variance. To quantify this stability independent of scale differences, we compute the coefficient of variation (CV) across five independent runs. ARISE achieves an average CV of 0.14, substantially lower than Scaling Metric's 0.28. This pattern extends to multimodal tasks (Figure~\ref{fig:multimodal_arise_scaling_estimation}), where ARISE's average CV of 0.08 outperforms Scaling Metric's 0.15. The enhanced stability stems from our adaptive sampling mechanism, which dynamically allocates computational budget based on observed variance patterns.

\paragraph{Cross-Dataset Consistency.} 
Table~\ref{tab:arise_scaling_processed} demonstrates ARISE's remarkable consistency across diverse evaluation domains. For instance, Claude Opus 4.1 maintains ARISE scores within a narrow range (average deviation < 0.015) across mathematical and scientific reasoning tasks, despite their distinct problem characteristics. To quantify cross-dataset consistency, we compute the coefficient of variation across different benchmarks for each model. ARISE achieves an average inter-dataset CV of 0.194, compared to 0.249 for traditional Scaling Metric. This consistency indicates that ARISE captures fundamental scaling properties that transcend specific task domains, providing a more robust assessment of model capabilities.

\paragraph{Model Evolution Tracking.} 
Our experiments reveal clear progression patterns across model generations, validating ARISE's sensitivity to architectural improvements. The OpenAI o-series progresses from o1 (average ARISE $\approx$ 0.134) to o3 ($\approx$ 0.239), representing a 78\% improvement. More dramatically, the Anthropic Claude series shows substantial gains from Claude Sonnet 4 ($\approx$ 0.079) through Claude Opus 4 ($\approx$ 0.268) to Claude Opus 4.1 ($\approx$ 0.465), achieving nearly 6× improvement. These patterns, consistently observed in both text-based (Table~\ref{tab:arise_scaling_processed}) and multimodal evaluations (Figure~\ref{fig:multimodal_arise_scaling_estimation}), 
demonstrate that recent model development has successfully enhanced test-time scaling capabilities and validate ARISE's effectiveness in tracking these advancements.

\subsection{Adaptive Sampling Effectiveness}
\label{subsec:adaptive_sampling}

\begin{figure*}[thp]
    \centering
    \includegraphics[width=\textwidth]{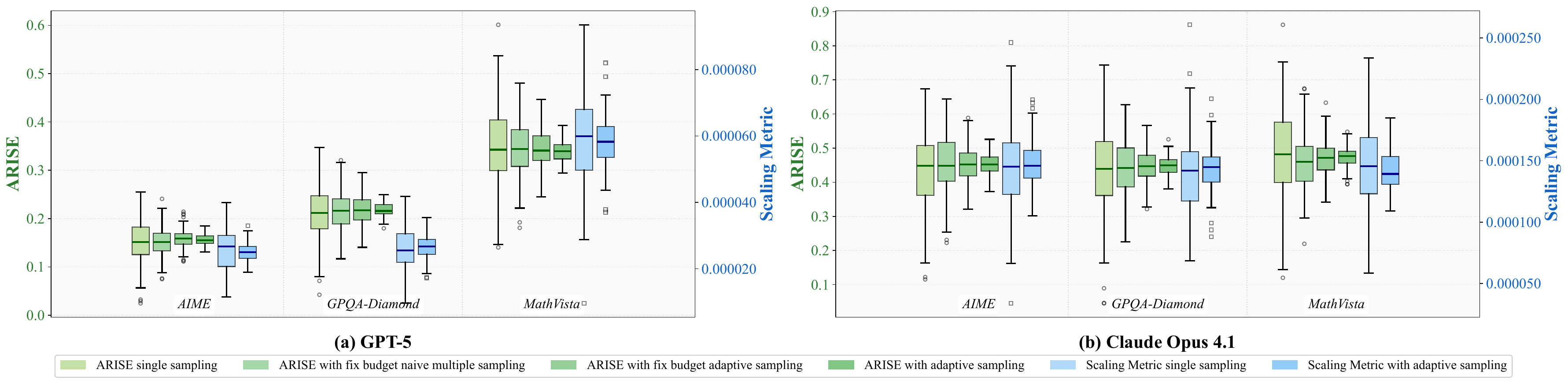}
    \vspace{-1.5em}
    \caption{Stability analysis of adaptive sampling on (a) GPT-5 and (b) Claude Opus 4.1. Box plots compare variance across six sampling strategies, demonstrating that adaptive sampling achieves superior stability.}
    \label{fig:adaptive_sampling_stability_analysis}
\end{figure*}

\begin{figure*}[thp]
    \centering
    \includegraphics[width=\textwidth]{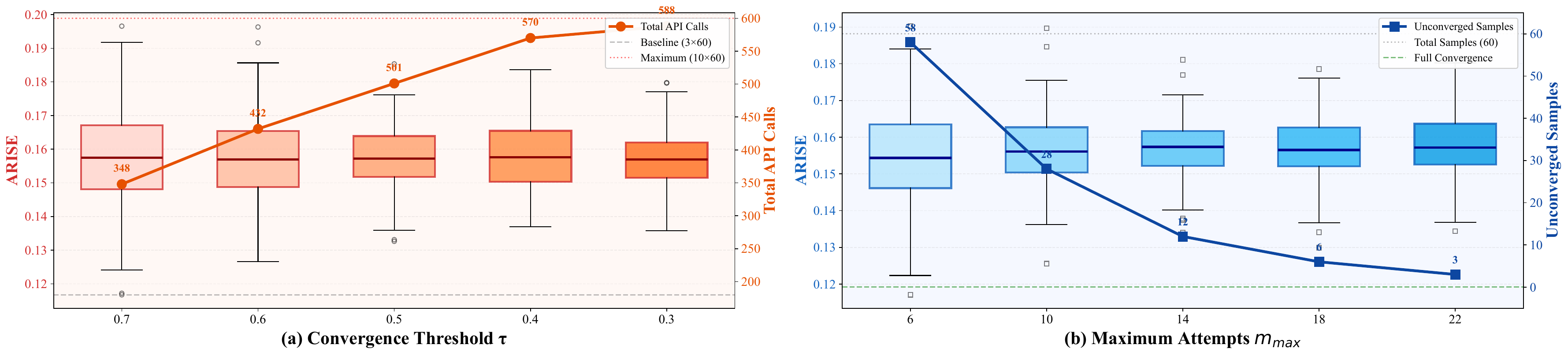}
    \vspace{-1.5em}
    \caption{Hyperparameter analysis of adaptive sampling using GPT-5 on AIME. (a) Threshold $\tau$: box plots show ARISE (left axis) while the line plot indicates total API calls (right axis). (b) Maximum attempts $m_{max}$: box plots display ARISE (left axis) while the line plot tracks unconverged samples failing to meet threshold $\tau$ (right axis).}
    \label{fig:adaptive_sampling_hyperparameter_analysis}
    \vspace{-1em}
\end{figure*}

\paragraph{Impact of Adaptive Sampling on Evaluation Stability.} 
Figure~\ref{fig:adaptive_sampling_stability_analysis} presents a comprehensive analysis of how adaptive sampling affects evaluation stability across different datasets and models. The results demonstrate that adaptive sampling substantially enhances stability for both metrics across all evaluated benchmarks. Notably, while ARISE exhibits comparable or even higher variance than Scaling Metric under single sampling, particularly on GPQA-Diamond, the introduction of adaptive sampling reverses this relationship. With full adaptive sampling enabled, ARISE achieves remarkable stability improvements with variance reduction of 76.1\%, significantly outperforming Scaling Metric's 54.2\% reduction. This superior improvement stems from ARISE's sample-level awareness, which enables more targeted allocation of computational resources to high-variance instances.

\vspace{-.5em}
\paragraph{Fixed Budget Analysis.} 
To evaluate the efficiency of adaptive sampling under resource constraints, we conduct experiments with fixed computational budgets, comparing adaptive sampling against naive multiple sampling that uniformly distributes resources across all instances. As illustrated in Figure~\ref{fig:adaptive_sampling_stability_analysis}, adaptive sampling consistently outperforms naive approaches at all configurations. Under fixed budget constraints, naive multiple sampling achieves a modest 31.4\% variance reduction compared to single sampling, while adaptive sampling delivers a substantial 57.5\% reduction, nearly doubling the effectiveness despite identical costs. This efficiency gain is particularly pronounced for challenging instances where performance exhibits high variability. The adaptive mechanism identifies these unstable cases and allocates additional samples accordingly, whereas naive sampling wastes resources on already-stable predictions.

\subsection{Hyperparameter Analysis}
\label{subsec:hyperparameter_analysis}

\paragraph{Convergence Threshold.}
Figure~\ref{fig:adaptive_sampling_hyperparameter_analysis}(a) examines the trade-off between evaluation stability and computational cost for convergence threshold $\tau$. As $\tau$ decreases, variance monotonically decreases while sampling counts increase substantially. However, diminishing returns become evident beyond $\tau = 0.5$, indicating stricter thresholds provide minimal stability gains at significantly higher computational cost. This establishes $\tau = 0.5$ as the optimal operating point, balancing effective variance reduction with reasonable resource consumption.

\vspace{-.5em}
\paragraph{Maximum Sampling Attempts.}
Figure~\ref{fig:adaptive_sampling_hyperparameter_analysis}(b) analyzes the effect of maximum attempts $m_{max}$ on convergence behavior. While increasing $m_{max}$ progressively reduces variance, the improvement rate declines after $m_{max} = 10$, as evidenced by the rapidly decreasing number of unconverged samples. Extended sampling beyond this threshold yields marginal benefits while potentially doubling computation for difficult instances. These findings confirm that $m_{max} = 10$ optimally balances stability and efficiency.


\section{Conclusion}
\label{sec:conclusion}

We propose ARISE, a novel metric for evaluating test-time scaling capabilities that addresses fundamental limitations of existing approaches through sample-level awareness and adaptive sampling mechanisms. By tracking individual sample trajectories across scaling iterations and appropriately penalizing performance degradation, ARISE provides principled assessment of scaling effectiveness while the dynamic sampling strategy ensures statistically reliable measurements under variance. Experiments across diverse reasoning tasks demonstrate that ARISE delivers consistent, discriminative and stable evaluations, establishing it as an principled framework for comparing test-time scaling capabilities across large reasoning models.
\section*{Limitations and Broader Impacts}
\label{sec:limitations}

\paragraph{Generalization to Diverse Scenarios.} While we conducted extensive evaluation across mathematical, scientific, coding, agentic, and multimodal tasks, the test-time scaling capabilities of reasoning models in more scenarios remain unexplored. For instance, more diverse agentic environments beyond the mock, airline, retail, and telecom scenarios in $\tau^2$-bench~\citep{barres2025tau2} is necessary to explore. Although expanding evaluation scope incurs additional computational costs, ARISE demonstrates strong cross-task consistency, suggesting that incorporating more evaluation domains will progressively approximate models' true test-time scaling capabilities.

\paragraph{Cross-Linguistic Evaluation.} Our evaluation is currently limited to English-language datasets due to resource constraints. Cross-linguistic analysis is essential as AI systems should equitably serve users across different languages rather than privileging English speakers. Previous research has identified significant performance disparities in large language models when solving identical problems across different languages~\citep{shafayat2024multifact,xuan2025mmluproxmultilingualbenchmarkadvanced}. Future work should evaluate test-time scaling behaviors of large reasoning models across diverse languages when appropriate evaluation datasets become available, ensuring that scaling benefits are universally accessible.

\section*{Ethics Statement}
\label{sec:ethics_statement}

\paragraph{Data Usage Compliance.} Throughout our experiments, we strictly adhere to all applicable data usage regulations and licensing requirements. Table~\ref{tab:dataset_statistics} provides comprehensive details for each dataset, including domain, answer type, sample size, and corresponding license information. All datasets utilized in this work are in English, and we have carefully verified that our usage complies with the specific licensing terms of each dataset.

\paragraph{Model Usage Compliance.} We maintain strict adherence to all model usage regulations and terms of service. Table~\ref{tab:model_statistics} details the specifications of each model, including parameter counts, release dates, and releasing organizations. For proprietary models, we strictly comply with the terms of service established by their respective organizations and exclusively utilize official APIs. For open-source models, we ensure full compliance with their corresponding licenses and usage guidelines.

\paragraph{Safety and Content Moderation.} We implement rigorous safeguards to ensure that all model outputs remain safe and appropriate, containing no personal information or offensive content. The datasets employed in our evaluation have undergone thorough community validation and review, and we have identified no instances of unsafe or inappropriate content throughout our experimental process.

\paragraph{Use of AI Assistants.} During the development of this work, we utilized Cursor for code composition and development. We ensure that all usage of AI tools strictly complies with submission guidelines and ethical standards established by the ACL community.

\bibliography{custom}

\appendix
\section{Boundedness Analysis}
\label{app:boundedness}

We provide formal boundedness proofs for both the traditional Scaling Metric and our proposed ARISE metric, demonstrating the fundamental differences in their mathematical properties.

\subsection{Scaling Metric Bounds}

The Scaling Metric~\citep{muennighoff2025s1simpletesttimescaling} is defined as:
\begin{equation}
\text{Scaling} = \frac{1}{\binom{|\mathcal{P}|}{2}} \sum_{\substack{p_1, p_2 \in \mathcal{P} \\ \mathcal{T}(p_2) > \mathcal{T}(p_1)}} \frac{\mathcal{A}(p_2) - \mathcal{A}(p_1)}{\mathcal{T}(p_2) - \mathcal{T}(p_1)}
\label{eq:scaling_app}
\end{equation}

\begin{theorem}[Scaling Metric Bounds]
\label{thm:scaling_bounds}
The Scaling Metric is bounded: $\text{Scaling} \in \left[\frac{-1}{\delta_{\min}}, \frac{1}{\delta_{\min}}\right]$ where $\delta_{\min} = \min_{p_1,p_2} (\mathcal{T}(p_2) - \mathcal{T}(p_1))$.
\end{theorem}

\begin{proof}
For any pair $(p_1, p_2)$ with $\mathcal{T}(p_2) > \mathcal{T}(p_1)$:

Since $\mathcal{A} \in [0,1]$, the numerator satisfies:
\begin{equation}
-1 \leq \mathcal{A}(p_2) - \mathcal{A}(p_1) \leq 1
\label{eq:num_bound}
\end{equation}

For the denominator, by definition:
\begin{equation}
\mathcal{T}(p_2) - \mathcal{T}(p_1) \geq \delta_{\min} > 0
\label{eq:denom_bound}
\end{equation}

Combining Equations~\ref{eq:num_bound} and \ref{eq:denom_bound}:
\begin{equation}
\frac{-1}{\delta_{\min}} \leq \frac{\mathcal{A}(p_2) - \mathcal{A}(p_1)}{\mathcal{T}(p_2) - \mathcal{T}(p_1)} \leq \frac{1}{\delta_{\min}}
\label{eq:pair_bound}
\end{equation}

Since the Scaling Metric is an average over all pairs, it inherits these bounds. In practice, with normalized token differences where $\delta_{\min} \to 1^+$, we obtain $\text{Scaling Metric} \in (-1, 1)$.
\end{proof}

\subsection{ARISE Bounds}

For ARISE, we analyze bounds considering all possible accuracy transition patterns.

\begin{theorem}[Sample-Level ARISE Bounds]
\label{thm:arise_sample_bounds}
For a single sample $i$ with $m$ scaling iterations, $\text{ARISE}_i \in (-\infty, 1)$.
\end{theorem}

\begin{proof}
From Equation~\ref{eq:cases}, ARISE$_i$ is the sum of contributions from all transitions. Let $\mathcal{I} = \{j: a_i^{(j)} = 1, a_i^{(j-1)} = 0\}$ denote improvement transitions and $\mathcal{D} = \{j: a_i^{(j)} = 0, a_i^{(j-1)} = 1\}$ denote degradation transitions.

\begin{equation}
\text{ARISE}_i = \sum_{j \in \mathcal{I}} \frac{t_i^{(j-1)}}{t_i^{(j)}} - \sum_{j \in \mathcal{D}} \frac{t_i^{(j)}}{t_i^{(j-1)}}
\label{eq:arise_decomp}
\end{equation}

\textbf{Upper bound:} The maximum occurs when there are only improvements and no degradations. The optimal accuracy sequence with a single transition at position $j^*$:

\begin{equation}
\text{ARISE}_{i,\max} = \frac{t_i^{(j^*-1)}}{t_i^{(j^*)}} < 1
\label{eq:single_improvement}
\end{equation}

As $t_i^{(j^*)} \to t_i^{(j^*-1)}$, ARISE$_{i,\max} \to 1^-$. Thus, the supremum is 1.

\textbf{Lower bound:} The minimum occurs with only degradations and no improvements. The worst accuracy sequence with a single transition at position $j^*$:

\begin{equation}
\text{ARISE}_{i,\min} = -\frac{t_i^{(j^*)}}{t_i^{(j^*-1)}} < -1
\label{eq:single_degradation}
\end{equation}

Since $t_i^{(j^*)}$ can be arbitrarily larger than $t_i^{(j^*-1)}$, ARISE$_{i,\min} \to -\infty$, yielding intermediate values within $(-\infty, 1)$.
\end{proof}

\begin{theorem}[ARISE Bounds]
\label{thm:arise_global_bounds}
The aggregate ARISE metric satisfies $\text{ARISE} \in (-\infty, 1)$.
\end{theorem}

\begin{proof}
From Equation~\ref{eq:arise_total}, ARISE is the arithmetic mean of sample scores:

\begin{equation}
\text{ARISE} = \frac{1}{n} \sum_{i=1}^{n} \text{ARISE}_i
\label{eq:global_arise}
\end{equation}

By Theorem~\ref{thm:arise_sample_bounds}, each $\text{ARISE}_i \in (-\infty, 1)$.

\textbf{Upper bound:} When all samples achieve the optimal pattern :
\begin{equation}
\text{ARISE}_{\max} = \lim_{t_i^{(j^*)} \to t_i^{(j^*-1)}} \frac{1}{n} \sum_{i=1}^n \frac{t_i^{(j^*-1)}}{t_i^{(j^*)}} < 1
\label{eq:global_max}
\end{equation}

\textbf{Lower bound:} When all samples exhibit the worst pattern :
\begin{equation}
\text{ARISE}_{\min} = \lim_{r \to \infty} \frac{1}{n} \sum_{i=1}^n (-r) = -\infty
\label{eq:global_min}
\end{equation}
where $r = \frac{t_i^{(j^*)}}{t_i^{(j^*-1)}}$ is the token ratio. Thus, the range of the ARISE is $(-\infty, 1)$.
\end{proof}

\section{Algorithm}
\label{app:algorithm}

We present the complete algorithmic procedure for computing ARISE with adaptive sampling. Algorithm~\ref{alg:arise} outlines the main computation flow, while Algorithm~\ref{alg:adaptive} details the dynamic sampling strategy.

\begin{algorithm}[t]
\caption{ARISE Computation with Adaptive Sampling}
\label{alg:arise}
\begin{algorithmic}[1]
\Require Dataset $\mathcal{D} = \{x_1, \ldots, x_n\}$, model $\mathcal{M}$, scaling iterations $J$, budget $B$
\Ensure ARISE score
\State Initialize $\text{ARISE} \gets 0$
\For{$i = 1$ to $n$} \Comment{For each sample}
    \State $\text{ARISE}_i \gets 0$
    \For{$j = 1$ to $J$} \Comment{Scaling iterations}
        \State // Adaptive sampling strategy
        \State $k^*, a_i^{(j)}, t_i^{(j)} \gets$ \Call{AdaptiveSample}{$\mathcal{M}, x_i, j$}
        \State // Compute contribution
        \State $\Delta a \gets a_i^{(j)} - a_i^{(j-1)}$
        \If{$\Delta a \neq 0$}
            \State $W \gets \left(\frac{t_i^{(j-1)}}{t_i^{(j)}}\right)^{\text{sign}(\Delta a)}$
            \State $\text{ARISE}_i \gets \text{ARISE}_i + \Delta a \cdot W$
        \EndIf
    \EndFor
    \State $\text{ARISE} \gets \text{ARISE} + \text{ARISE}_i$
\EndFor
\State \Return $\text{ARISE} / n$
\end{algorithmic}
\end{algorithm}

\begin{algorithm}[t]
\caption{Adaptive Sampling Strategy}
\label{alg:adaptive}
\begin{algorithmic}[1]
\Require Model $\mathcal{M}$, sample $x_i$, iteration $j$, threshold $\tau$, max trials $m_{\max}$
\Ensure Sample count $k^*$, mean accuracy $a_i^{(j)}$, mean tokens $t_i^{(j)}$
\State // Initial probing phase
\State $k \gets m_{\min}$ \Comment{Minimum trials}
\For{$\ell = 1$ to $m_{\min}$}
    \State $a_{i,\ell}^{(j)} \gets$ \Call{Evaluate}{$\mathcal{M}, x_i, t^{(j)}$}
    \State $t_{i,\ell}^{(j)} \gets$ \Call{CountTokens}{$\mathcal{M}, x_i$}
\EndFor
\State // Compute initial statistics
\State $\mu_a \gets \frac{1}{k} \sum_{\ell=1}^{k} a_{i,\ell}^{(j)}$
\State $\sigma_a \gets \sqrt{\frac{1}{k} \sum_{\ell=1}^{k} (a_{i,\ell}^{(j)} - \mu_a)^2}$
\State $\mu_t \gets \frac{1}{k} \sum_{\ell=1}^{k} t_{i,\ell}^{(j)}$
\State $\sigma_t \gets \sqrt{\frac{1}{k} \sum_{\ell=1}^{k} (t_{i,\ell}^{(j)} - \mu_t)^2}$
\State // Compute coefficient of variation
\State $\text{CV} \gets \frac{\sigma_a}{\mu_a + \epsilon} + \frac{\sigma_t}{\mu_t + \epsilon}$
\While{$\text{CV} \ge \tau$ \textbf{and} $k \le m_{\max}$}
    \State $k \gets k + 1$
    \State $a_{i,k}^{(j)} \gets$ \Call{Evaluate}{$\mathcal{M}, x_i, t^{(j)}$}
    \State $t_{i,k}^{(j)} \gets$ \Call{CountTokens}{$\mathcal{M}, x_i$}
    \State Update $\mu_a, \sigma_a, \mu_t, \sigma_t, \text{CV}$
\EndWhile
\State $k^* \gets k$ \Comment{Final sample count}
\State $a_i^{(j)} \gets \frac{1}{k^*} \sum_{\ell=1}^{k^*} a_{i,\ell}^{(j)}$
\State $t_i^{(j)} \gets \frac{1}{k^*} \sum_{\ell=1}^{k^*} t_{i,\ell}^{(j)}$
\State \Return $k^*, a_i^{(j)}, t_i^{(j)}$
\end{algorithmic}
\end{algorithm}

\subsection{Algorithmic Analysis}

\paragraph{Computational Complexity.}
The time complexity of ARISE is $\mathcal{O}(n \cdot J \cdot \bar{k} \cdot C_{\text{eval}})$, where $n$ is the dataset size, $J$ is the number of scaling iterations, $\bar{k}$ is the average sample count per configuration, and $C_{\text{eval}}$ is the cost of a single model evaluation. The adaptive termination in Algorithm~\ref{alg:adaptive} ensures $\bar{k} \ll m_{\max}$ for stable configurations, significantly reducing computation compared to fixed sampling.

\paragraph{Early Termination.}
Algorithm~\ref{alg:adaptive} implements dynamic early stopping based on the coefficient of variation (CV) from accuracy and token statistics (line 14). When $\text{CV} < \tau$, sampling terminates due to sufficient statistical stability. Low-variance configurations converge with $k \approx m_{\min}$ samples, while high-variance cases automatically expand to $k \to m_{\max}$.

\paragraph{Adaptive Resource Allocation.}
The sample count $k^*$ varies per configuration based on observed variance, enabling natural budget redistribution. Stable configurations converge with minimal samples, while stochastic configurations adaptively require more trials, concentrating resources where uncertainty is highest.
\section{Experiment Details}
\label{app:experiment_details}

\subsection{Inference Configuration}

For open-source models, we deploy them using vLLM~\citep{kwon2023efficient} for efficient inference. Following official recommendations, we configure the generation hyperparameters as follows: Temperature = 0.6, Top-P = 0.95, Top-K = 20, and Min-P = 0. These settings balance between generation diversity and output quality, enabling controlled yet flexible reasoning processes during test-time scaling evaluation. 

\subsection{Test-Time Scaling Configuration}

For models supporting variable test-time computation, we evaluate across multiple scaling levels:
\begin{itemize}
    \item \textbf{Effort-based scaling} (OpenAI o-series, gpt-oss models): We test \texttt{low}, \texttt{medium}, and \texttt{high} reasoning effort levels, allowing models to allocate increasing computational resources to complex problems.
    \item \textbf{Mode-based scaling} (Claude 4 series, Qwen3 series and DeepSeek models): We compare \texttt{think} mode (with explicit reasoning chains) against \texttt{no-think} mode (direct response generation). For DeepSeek models, we utilize \texttt{reasoner} (\texttt{think}) and standard \texttt{chat} (\texttt{no-think}) modes for comparative analysis, with explicit prompts employed to control reasoning length for DeepSeek-R1.
\end{itemize}
All experiments are conducted with a maximum token limit of 32,768 for reasoning traces and 2,048 for queries, ensuring sufficient space for complex multi-step reasoning while maintaining computational feasibility. For consistency, we perform five independent runs for each configuration and report averaged results with standard deviations.

\section{Dataset Details}
\label{app:dataset_details}
\begin{table*}[t]
\centering

\footnotesize
\resizebox{\textwidth}{!}{%
\begin{tabular}{l|l|c|c|c}
\toprule
\textsc{Dataset} & \textsc{Domain} & \textsc{Answer Format} & \textsc{\# Samples} & \textsc{License} \\
\midrule
\multicolumn{5}{c}{\textit{Mathematical Reasoning}} \\
\midrule
\href{https://huggingface.co/datasets/MathArena/aime_2025}{AIME}~\citep{aime} & Competition Math & Integer (0-999) & 60 & CC BY-NC-SA 4.0 \\
\href{https://github.com/matharena/hmmt}{HMMT}~\citep{balunovic2025srimatharena} & Competition Math & Free-form & 60 & CC BY-NC-SA 4.0 \\
\midrule
\multicolumn{5}{c}{\textit{Scientific Reasoning}} \\
\midrule
\href{https://github.com/idavidrein/gpqa}{GPQA-Diamond}~\citep{rein2024gpqa} & Graduate Science & Multi-Choice (4) & 198 & CC BY 4.0 \\
\href{https://huggingface.co/datasets/TIGER-Lab/MMLU-Pro}{MMLU-Pro}~\citep{wang2024mmlu} & Professional Knowledge & Multi-Choice (10) & 12,032 & MIT \\
\midrule
\multicolumn{5}{c}{\textit{Code Generation}} \\
\midrule
\href{https://huggingface.co/datasets/princeton-nlp/SWE-bench_Verified}{SWE-bench Verified}~\citep{Jimenez2024swe} & Software Engineering & Code Patch & 500 & CC0 1.0 \\
\href{https://huggingface.co/datasets/livecodebench/code_generation}{LiveCodeBench}~\citep{jain2025livecodebench} & Code Generation & Full Code & 121 & CC \\
\midrule
\multicolumn{5}{c}{\textit{Agentic Tasks}} \\
\midrule
\href{https://github.com/sierra-research/tau2-bench}{$\tau^2$-Bench}~\citep{barres2025tau2} & Conversational Agents & Tool Calls & 279 & MIT \\
\href{https://gorilla.cs.berkeley.edu/blogs/13_bfcl_v3_multi_turn.html}{BFCL-v3}~\citep{patil2025bfcl} & Function Calling & JSON & 4,441 & Apache 2.0 \\
\midrule
\multicolumn{5}{c}{\textit{Multimodal Reasoning}} \\
\midrule
\href{https://huggingface.co/datasets/MMMU/MMMU}{MMMU}~\citep{yue2024mmmu} & Multimodal Understanding & Multi-Choice & 11,550 & Apache 2.0 \\
\href{https://huggingface.co/datasets/AI4Math/MathVista}{MathVista}~\citep{lu2024mathvista} & Mathematical Visual & Mixed & 6,141 & CC BY-SA 4.0 \\
\href{https://huggingface.co/datasets/princeton-nlp/CharXiv}{CharXiv-Reasoning}~\citep{wang2024charxiv} & Chart Understanding & Free-form & 2,323 & CC BY-SA 4.0 \\
\bottomrule
\end{tabular}
}
\caption{Statistics of evaluation datasets used in our ARISE experiments. Datasets span mathematical reasoning, scientific understanding, code generation, agentic tasks, and multimodal reasoning domains.}
\label{tab:dataset_statistics}
\vspace{-1em}
\end{table*}

In our experiments, we selected eleven datasets encompassing a comprehensive range of task types that require sophisticated reasoning capabilities and test-time scaling behaviors. These datasets were specifically chosen to evaluate different aspects of reasoning model performance across mathematical, scientific, coding, agentic, and multimodal domains. Detailed statistics for these datasets are provided in Table~\ref{tab:dataset_statistics}.

\subsection{Mathematical Reasoning Datasets}

\begin{itemize}
\item \textbf{AIME}~\citep{aime} (American Invitational Mathematics Examination) consists of 60 challenging mathematics problems selected from the 2024 and 2025 competitions (30 problems each year). Each problem requires an integer answer between 0 and 999, enabling precise evaluation of mathematical reasoning using Pass@1 as the metric. The problems span advanced topics including algebra, geometry, number theory, combinatorics, and probability, requiring multi-step reasoning and creative problem-solving approaches.

\item \textbf{HMMT}~\citep{balunovic2025srimatharena} (Harvard-MIT Mathematics Tournament) comprises 60 problems from the February 2024 and February 2025 competitions. Unlike AIME's integer-only format, HMMT accepts free-form mathematical expressions including fractions, algebraic expressions, and LaTeX-formatted answers, evaluated using Pass@1 metric. The dataset is categorized by problem type (combinatorics, algebra, geometry, and number theory).
\end{itemize}

\subsection{Scientific Reasoning Datasets}

\begin{itemize}
\item \textbf{GPQA-Diamond}~\citep{rein2024gpqa} (Graduate-Level Google-Proof Q\&A) contains 198 meticulously curated graduate-level science questions designed to be "Google-proof"—questions that cannot be easily answered through simple web searches. Each multiple-choice question has 4 options and was authored by PhD holders and validated by domain experts who achieved 81\% accuracy, while non-experts with internet access reached only 22\% accuracy after 30+ minutes of searching. The questions span biology, physics, and chemistry at the graduate level, with performance measured by accuracy.

\item \textbf{MMLU-Pro}~\citep{wang2024mmlu} represents a substantial enhancement over the original MMLU benchmark, containing 12,032 questions with 10 answer choices (A-J) instead of the traditional 4, reducing random guess probability from 25\% to 10\%. The dataset covers 14 domains including STEM fields, humanities, social sciences, business, health, and law.
\end{itemize}

\subsection{Code Generation Datasets}

\begin{itemize}
\item \textbf{SWE-bench Verified}~\citep{Jimenez2024swe} provides 500 human-validated GitHub issues from 12 popular Python repositories, requiring models to generate code patches in git diff format to resolve real-world software engineering problems. Each issue is categorized by difficulty (15-minute "easy" to 1+ hour "hard" tasks) and has been verified by human developers to ensure solvability and clear problem specifications. The benchmark tests practical software engineering skills including debugging, feature implementation, and code refactoring, with performance measured by resolve rate.

\item \textbf{LiveCodeBench}~\citep{jain2025livecodebench} offers 121 coding problems continuously collected from competitive programming platforms (LeetCode, AtCoder, CodeForces) after major models' training cutoffs, ensuring contamination-free evaluation. Problems require complete, executable Python code generation and are tested using Pass@1. The dynamic nature of the dataset, with problems added monthly, makes it particularly valuable for evaluating true generalization capabilities.
\end{itemize}

\subsection{Agentic Task Datasets}

\begin{itemize}
\item \textbf{$\tau^2$-Bench}~\citep{barres2025tau2} evaluates conversational agents through 279 multi-turn dialogues across retail, airline, and telecom domains. The benchmark uniquely tests dual-control scenarios where both the agent and simulated user can modify shared state through tool calls and API interactions. Success is measured through Pass$^k$ (Pass-Hat-K) metric assessing task completion rates, policy adherence, and database state matching, providing comprehensive assessment of agent coordination and planning capabilities.

\item \textbf{BFCL-v3}~\citep{patil2025bfcl} (Berkeley Function Calling Leaderboard v3) contains 4,441 function calling scenarios testing single-turn, multi-turn, and multi-step interactions. Models must generate properly formatted JSON function calls across Python, Java, and JavaScript, with evaluation using Abstract Syntax Tree (AST) substring matching for single-turn scenarios and combined state-based and response-based evaluation for multi-turn entries. The benchmark includes relevance detection tasks where models must determine when not to call functions, testing both precision and recall in tool use.
\end{itemize}

\begin{figure*}[thp]
    \centering
    \includegraphics[width=\textwidth]{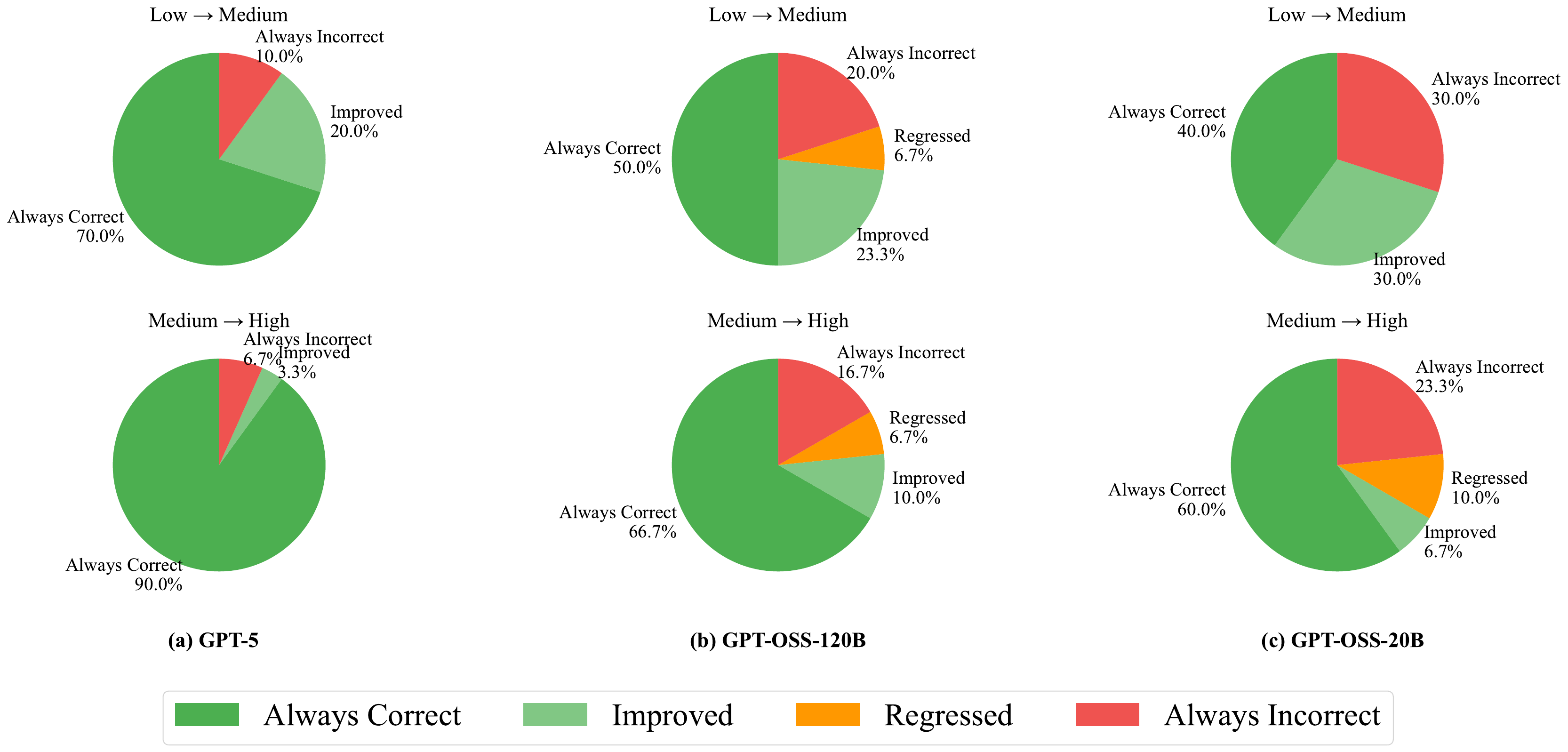}
    \caption{Sample-level accuracy transitions during test-time scaling on AIME dataset. We track individual sample accuracy changes across scaling iterations for GPT-5, GPT-OSS-120B, and GPT-OSS-20B. GPT-OSS models exhibit substantial performance degradation where many initially correct samples become incorrect.}
    \label{fig:aime2025_performance_comparison}
    \vspace{-1em}
\end{figure*}

\subsection{Multimodal Reasoning Datasets}

\begin{itemize}
\item \textbf{MMMU}~\citep{yue2024mmmu} (Massive Multi-discipline Multimodal Understanding) comprises 11,550 college-level questions requiring joint visual and textual reasoning across 30 subjects and 183 subfields. The dataset incorporates over 30 heterogeneous image types including charts, diagrams, maps, tables, chemical structures, and music sheets. Questions span six core disciplines (Art \& Design, Business, Science, Health \& Medicine, Humanities \& Social Science, Tech \& Engineering), testing expert-level multimodal understanding with micro-averaged accuracy as the evaluation metric.

\item \textbf{MathVista}~\citep{lu2024mathvista} provides 6,141 mathematical reasoning problems in visual contexts, supporting both multiple-choice and free-form answers (integer, float, text, or list) evaluated by accuracy. The benchmark tests seven reasoning types: algebraic, arithmetic, geometric, logical, numeric common sense, scientific, and statistical reasoning. Problems are derived from 28 existing datasets plus three newly created sources, emphasizing mathematical reasoning grounded in visual information.

\item \textbf{CharXiv-Reasoning}~\citep{wang2024charxiv} consists of 2,323 high-resolution scientific charts extracted from arXiv preprints, generating 11,615 questions (2,323 reasoning-focused). The dataset emphasizes complex chart understanding requiring synthesis across multiple visual elements, trend analysis, and comparative reasoning, with accuracy as the evaluation metric. Unlike simpler chart QA datasets, CharXiv features realistic scientific visualizations with complex legends, multiple subplots, and domain-specific annotations typical of academic publications.
\end{itemize}

\section{Model Details}
\label{app:model_details}

We evaluate 22 state-of-the-art large reasoning models spanning four major organizations: OpenAI, Anthropic, Alibaba (Qwen), and DeepSeek. These models represent the current frontier of test-time scaling capabilities, ranging from compact 0.6B parameter models to massive 671B parameter mixture-of-experts architectures. Comprehensive specifications including parameter counts, release dates, context windows, and organizational affiliations are provided in Table~\ref{tab:model_statistics}. The evaluated models employ two primary test-time scaling approaches:

\begin{itemize}
    \item \textbf{Effort-based scaling}: Models dynamically adjust computational resources during inference through configurable reasoning effort levels (\texttt{low}, \texttt{medium}, \texttt{high}). This approach, exclusive to OpenAI's o-series models (o1, o3, o3-mini, o4-mini) and open-weight gpt-oss variants (20B, 120B), allows users to explicitly control the depth of reasoning based on task complexity and latency requirements.
    
    \item \textbf{Mode-based scaling}: Models switch between distinct reasoning modes through specialized prompts. All non-OpenAI models in our evaluation, including Anthropic's Claude 4 series, Alibaba's Qwen3 family, and DeepSeek's V3.1, V3.1-Terminus and V3.2-Exp, utilize this approach. These models toggle between standard response generation (\texttt{no-think} or \texttt{chat} mode) and enhanced reasoning with explicit chain-of-thought traces (\texttt{think} or \texttt{reasoner} mode).
\end{itemize}
The diversity in model architectures, from dense transformers to mixture-of-experts, combined with varying test-time scaling mechanisms, provides a comprehensive landscape for evaluating the effectiveness of our proposed ARISE metric across different computational paradigms and reasoning strategies.
\section{Further Analysis}
\label{app:analysis}

\begin{figure*}[thp]
    \centering
    \includegraphics[width=\textwidth]{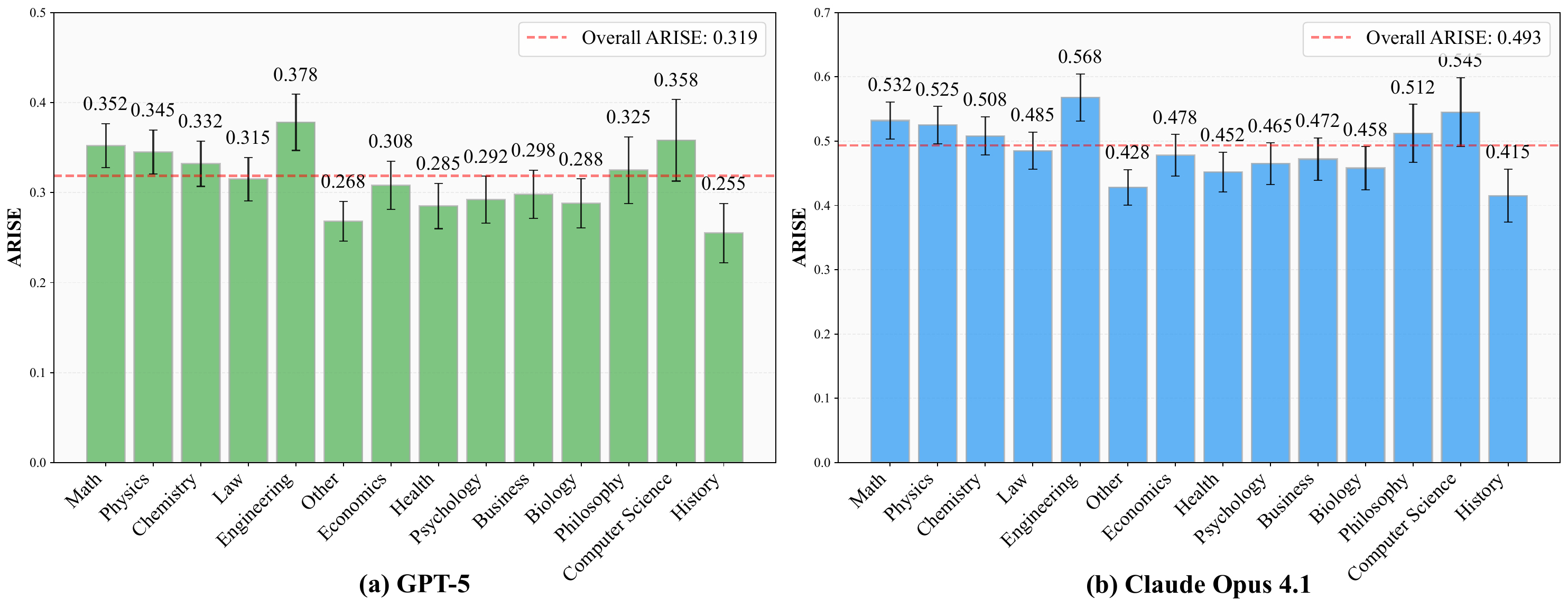}
    \caption{Category-level ARISE analysis across MMLU-Pro disciplines. Evaluation of (a) GPT-5 and (b) Claude Opus 4.1 on 14 distinct knowledge domains demonstrates consistent scaling patterns. Error bars represent standard errors. The red dashed line indicates the overall ARISE score across all categories.}
    \label{fig:mmlu_pro_category_analysis}
    \vspace{-1em}
\end{figure*}

\paragraph{Effective Capture of Performance Degradation During Scaling.} 
Figure~\ref{fig:aime2025_performance_comparison} presents our analysis of sample-level accuracy transitions for GPT series models on the AIME dataset. Through tracking individual sample trajectories across scaling levels, we identify severe performance degradation in GPT-OSS-120B and GPT-OSS-20B. Specifically, GPT-OSS-120B exhibits 6.7\% of samples transitioning from correct to incorrect when scaling from low to medium computational budget, with another 6.7\% degrading during the medium-to-high transition. GPT-OSS-20B demonstrates even more pronounced instability, with 10.0\% of samples experiencing degradation during medium-to-high scaling. In contrast, GPT-5 maintains consistent improvement throughout the scaling process without a single instance of correct samples becoming incorrect. These observations directly explain the significant negative ARISE scores observed for GPT-OSS-120B and GPT-OSS-20B in Table~\ref{tab:arise_scaling_processed}, demonstrating ARISE's unique ability to capture performance degradation that Scaling Metric completely overlooks.

\paragraph{Cross-Task Consistency.}
To further validate ARISE's consistency across different tasks, we conduct fine-grained analysis across MMLU-Pro's 14 discipline categories, spanning diverse domains from technical sciences to humanities. As illustrated in Figure~\ref{fig:mmlu_pro_category_analysis}, both GPT-5 and Claude Opus 4.1 exhibit remarkably consistent scaling patterns despite substantial task heterogeneity. Technical and reasoning-intensive disciplines consistently achieve the highest ARISE values, with Engineering leading for both models (GPT-5: 0.378, Opus 4.1: 0.568), followed by Computer Science (0.358, 0.545) and Mathematics (0.352, 0.532). This pattern suggests that complex problem-solving tasks with multiple solution paths benefit most substantially from test-time computation. Furthermore, the cross-category coefficient of variation remains below 0.15 for both models, confirming ARISE's consistency across diverse knowledge domains.

\begin{figure}[thp]
    \centering
    \includegraphics[width=0.5\textwidth]{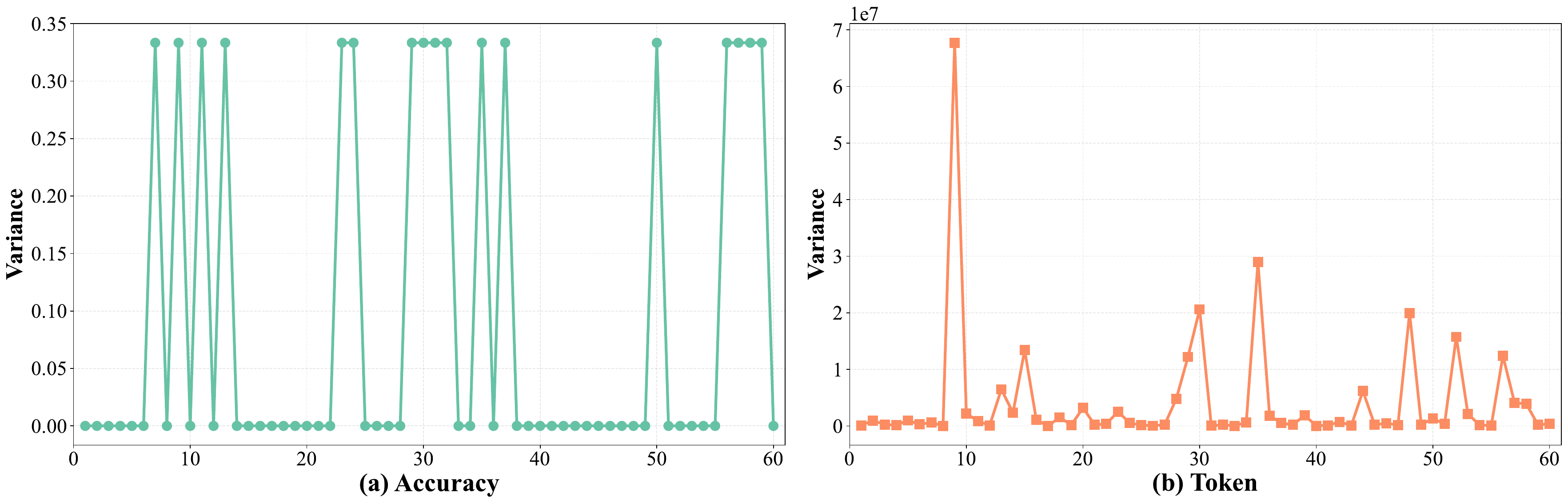}
    \caption{Variance analysis of sample accuracy and completion token counts during test-time scaling. Comparison of (a) accuracy variance and (b) token count variance for GPT-5 evaluated on the AIME dataset across different scaling levels. The completion token variance exhibits an order of magnitude of $10^7$, significantly exceeding the accuracy variance.}
    \label{fig:variance_analysis}
    \vspace{-1em}
\end{figure}

\paragraph{Variance Analysis of Accuracy and Token Counts} 
Figure~\ref{fig:variance_analysis} presents the analysis of the variance in sample accuracy and completion token counts. On the AIME dataset, we observe notable fluctuations in accuracy across samples, which arise from the sampling of both correct and incorrect reasoning paths. While completion token variance exhibits substantial fluctuations in only a small subset of samples, its magnitude is several orders higher than that of accuracy variance. To address this scale disparity, we employ the coefficient of variation (CV) in Equation~\ref{eq:cv}, which normalizes the variance by the mean. Our analysis on the AIME dataset yields $\text{CV}_a = 0.3175$ and $\text{CV}_t = 0.2854$ for GPT-5, demonstrating comparable magnitudes after normalization. This numerical proximity justifies our design choice in Equation~\ref{eq:combined_cv}, where we adopt a summation approach to compute the combined uncertainty metric that integrates both accuracy and token count variability.

\begin{table*}[t]
\centering

\footnotesize
\resizebox{\textwidth}{!}{%
\begin{tabular}{l|c|c|c|c|c}
\toprule
\textsc{Model} & \textsc{\# Parameters} & \textsc{Release Date} & \textsc{Context Window} & \textsc{Scaling Type} & \textsc{Organization} \\
\midrule
\multicolumn{6}{c}{\textit{OpenAI Models}} \\
\midrule
\href{https://openai.com/o1/}{o1}~\citep{openai2024o1} & / & 2024-09 & 128K & Effort & OpenAI \\
\href{https://openai.com/index/introducing-o3-and-o4-mini/}{o3}~\citep{openai2025o3} & / & 2025-04 & 200K & Effort & OpenAI \\
\href{https://openai.com/index/openai-o3-mini/}{o3-mini}~\citep{openai2025o3mini} & / & 2025-01 & 200K & Effort & OpenAI \\
\href{https://openai.com/index/introducing-o3-and-o4-mini/}{o4-mini}~\citep{openai2025o4mini} & / & 2025-04 & 200K & Effort & OpenAI \\
\href{https://openai.com/index/gpt-5-new-era/}{GPT-5}~\citep{openai2025gpt5} & / & 2025-08 & 400K & Effort & OpenAI \\
\href{https://github.com/openai/gpt-oss}{gpt-oss-20B}~\citep{openai2025gptoss} & 21B (3.6B active) & 2025-08 & 128K & Effort & OpenAI \\
\href{https://github.com/openai/gpt-oss}{gpt-oss-120B}~\citep{openai2025gptoss} & 117B (5.1B active) & 2025-08 & 128K & Effort & OpenAI \\
\midrule
\multicolumn{6}{c}{\textit{Anthropic Models}} \\
\midrule
\href{https://www.anthropic.com/news/claude-3-7-sonnet}{Claude Sonnet 4}~\citep{anthropic2025claude4} & / & 2025-05 & 200K & Mode & Anthropic \\
\href{https://www.anthropic.com/news/claude-4}{Claude Opus 4}~\citep{anthropic2025claude4} & / & 2025-05 & 200K & Mode & Anthropic \\
\href{https://www.anthropic.com/news/claude-opus-4-1}{Claude Opus 4.1}~\citep{anthropic2025claude41} & / & 2025-08 & 200K & Mode & Anthropic \\
\midrule
\multicolumn{6}{c}{\textit{Qwen Models}} \\
\midrule
\href{https://huggingface.co/Qwen/Qwen3-0.6B}{Qwen3-0.6B}~\citep{yang2025qwen3} & 0.6B & 2025-04 & 32K & Mode & Alibaba \\
\href{https://huggingface.co/Qwen/Qwen3-1.7B}{Qwen3-1.7B}~\citep{yang2025qwen3} & 1.7B & 2025-04 & 32K & Mode & Alibaba \\
\href{https://huggingface.co/Qwen/Qwen3-4B}{Qwen3-4B}~\citep{yang2025qwen3} & 4B & 2025-04 & 32K & Mode & Alibaba \\
\href{https://huggingface.co/Qwen/Qwen3-8B}{Qwen3-8B}~\citep{yang2025qwen3} & 8B & 2025-04 & 128K & Mode & Alibaba \\
\href{https://huggingface.co/Qwen/Qwen3-14B}{Qwen3-14B}~\citep{yang2025qwen3} & 14B & 2025-04 & 128K & Mode & Alibaba \\
\href{https://huggingface.co/Qwen/Qwen3-32B}{Qwen3-32B}~\citep{yang2025qwen3} & 32B & 2025-04 & 128K & Mode & Alibaba \\
\href{https://huggingface.co/Qwen/Qwen3-30B-A3B}{Qwen3-30B-A3B}~\citep{yang2025qwen3} & 30B (3B active) & 2025-04 & 128K & Mode & Alibaba \\
\href{https://huggingface.co/Qwen/Qwen3-235B-A22B}{Qwen3-235B-A22B}~\citep{yang2025qwen3} & 235B (22B active) & 2025-04 & 256K & Mode & Alibaba \\
\midrule
\multicolumn{6}{c}{\textit{DeepSeek Models}} \\
\midrule
\href{https://huggingface.co/deepseek-ai/DeepSeek-R1}{DeepSeek-R1}~\citep{deepseekai2025deepseekr1} & 671B (37B active) & 2025-01 & 128K & Mode & DeepSeek \\
\href{https://huggingface.co/deepseek-ai/DeepSeek-V3.1}{DeepSeek-V3.1}~\citep{deepseekai2024deepseekv3} & 671B (37B active) & 2025-08 & 128K & Mode & DeepSeek \\
\href{https://huggingface.co/deepseek-ai/DeepSeek-V3.1-Terminus}{DeepSeek-V3.1-Terminus}~\citep{deepseekai2024deepseekv3} & 671B (37B active) & 2025-09 & 128K & Mode & DeepSeek \\
\href{https://huggingface.co/deepseek-ai/DeepSeek-V3.2-Exp}{DeepSeek-V3.2-Exp}~\citep{deepseekai2024deepseekv3} & 671B (37B active) & 2025-09 & 128K & Mode & DeepSeek \\
\bottomrule
\end{tabular}
}
\caption{Statistics of large reasoning models. Models are categorized by their test-time scaling approach: effort-based (adjustable reasoning levels), and mode-based (thinking mode switching).}
\label{tab:model_statistics}
\vspace{-1em}
\end{table*}
\begin{table*}[htbp]
\centering
\footnotesize
\begin{tabular}{l@{\hspace{0.3em}}rr@{\hspace{0.5em}}rr@{\hspace{0.5em}}rr@{\hspace{0.5em}}rr}
\toprule
\multirow{2}{*}{\textbf{Model}} & \multicolumn{2}{c}{\textbf{AIME}} & \multicolumn{2}{c}{\textbf{HMMT}} & \multicolumn{2}{c}{\textbf{GPQA Diamond}} & \multicolumn{2}{c}{\textbf{MMLU-Pro}} \\
\cmidrule(lr){2-3} \cmidrule(lr){4-5} \cmidrule(lr){6-7} \cmidrule(lr){8-9}
 & ARISE & SM & ARISE & SM & ARISE & SM & ARISE & SM \\
\midrule
o1                      & 0.134563 & 0.000060 & 0.127678 & 0.000018 & 0.122789 & 0.000038 & 0.150899 & 0.000050 \\
o3                      & 0.299253 & 0.000078 & 0.167327 & 0.000018 & 0.212377 & 0.000058 & 0.278925 & 0.000056 \\
o3-mini                 & 0.130586 & 0.000037 & 0.188770 & 0.000030 & 0.164882 & 0.000022 & 0.166285 & 0.000041 \\
o4-mini                 & 0.240192 & 0.000034 & 0.167306 & 0.000043 & 0.199441 & 0.000042 & 0.208032 & 0.000039 \\
\midrule
gpt-oss-20B            & -0.402954 & 0.000020 & -0.312641 & 0.000016 & -0.327418 & 0.000022 & -0.269428 & 0.000022 \\
gpt-oss-120B           & -0.333963 & 0.000027 & -0.199902 & 0.000024 & -0.273421 & 0.000028 & -0.161542 & 0.000031 \\
gpt-5                  & 0.156599 & 0.000025 & 0.299576 & 0.000026 & 0.218496 & 0.000026 & 0.318572 & 0.000029 \\
\midrule
Claude Sonnet 4        & 0.104098 & 0.000046 & 0.040367 & 0.000010 & 0.063612 & 0.000032 & 0.105911 & 0.000026 \\
Claude Opus 4          & 0.347516 & 0.000065 & 0.171656 & 0.000061 & 0.221169 & 0.000048 & 0.332988 & 0.000067 \\
Claude Opus 4.1        & \textbf{0.452939} & \textbf{0.000146} & \textbf{0.470892} & \textbf{0.000141} & \textbf{0.445388} & \textbf{0.000141} & \textbf{0.493213} & \textbf{0.000203} \\
\midrule
Qwen-3-0.6B            & 0.293593 & 0.000002 & 0.176909 & 0.000007 & 0.211447 & 0.000005 & 0.271553 & 0.000003 \\
Qwen-3-1.7B            & 0.365808 & 0.000027 & 0.274632 & 0.000025 & 0.323745 & 0.000032 & 0.391681 & 0.000034 \\
Qwen-3-4B              & 0.216582 & 0.000049 & 0.221653 & 0.000037 & 0.239971 & 0.000032 & 0.256920 & 0.000061 \\
Qwen-3-8B              & 0.308547 & 0.000034 & 0.301039 & 0.000026 & 0.302231 & 0.000035 & 0.327433 & 0.000035 \\
Qwen-3-14B             & 0.324731 & 0.000068 & 0.221253 & 0.000047 & 0.259408 & 0.000044 & 0.312000 & 0.000080 \\
Qwen-3-32B             & 0.388342 & 0.000076 & 0.203790 & 0.000058 & 0.264353 & 0.000057 & 0.399234 & 0.000065 \\
Qwen3-30B-A3B          & 0.329273 & 0.000050 & 0.385471 & 0.000074 & 0.372684 & 0.000043 & 0.416213 & 0.000069 \\
Qwen3-235B-A22B        & 0.391492 & 0.000081 & 0.430555 & 0.000041 & 0.406943 & 0.000076 & 0.453300 & 0.000079 \\
\midrule
Deepseek-R1            & -0.031810 & 0.000007 & -0.045531 & 0.000004 & -0.049256 & 0.000003 & -0.010844 & 0.000002 \\
V3.1          & 0.396647 & 0.000035 & 0.204811 & 0.000036 & 0.271420 & 0.000042 & 0.355928 & 0.000036 \\
V3.1-Terminus & 0.324070 & 0.000023 & 0.283478 & 0.000035 & 0.309142 & 0.000031 & 0.322799 & 0.000031 \\
V3.2-Exp     & 0.302922 & 0.000027 & 0.265136 & 0.000033 & 0.272557 & 0.000020 & 0.321917 & 0.000029 \\
\bottomrule
\end{tabular}
\vspace{-.5em}
\caption{Performance of mainstream models in mathematical and scientific reasoning. Each benchmark shows ARISE scores and corresponding Scaling Metrics (SM). V3.1, V3.1-Terminus, and V3.2-Exp are all DeepSeek series models. The Scaling Metric shows relatively small values with the first three digits typically being zeros.}
\label{tab:arise_scaling}
\vspace{-1.5em}
\end{table*}
\begin{table*}[htbp]
\centering
\footnotesize
\begin{tabular}{l@{\hspace{0.3em}}rr@{\hspace{0.5em}}rr@{\hspace{0.5em}}rr@{\hspace{0.5em}}rr}
\toprule
\multirow{2}{*}{\textbf{Model}} & \multicolumn{2}{c}{\textbf{SWE-bench Verified}} & \multicolumn{2}{c}{\textbf{LiveCodeBench}} & \multicolumn{2}{c}{\textbf{$\tau$2-Bench}} & \multicolumn{2}{c}{\textbf{BFCL-v3}} \\
\cmidrule(lr){2-3} \cmidrule(lr){4-5} \cmidrule(lr){6-7} \cmidrule(lr){8-9}
 & ARISE & SM & ARISE & SM & ARISE & SM & ARISE & SM \\
\midrule
o1                      & 0.120531 & 0.000048 & 0.104464 & 0.000025 & 0.097844 & 0.000016 & 0.090596 & 0.000019 \\
o3                      & 0.242285 & 0.000041 & 0.227926 & 0.000033 & 0.188289 & 0.000025 & 0.163869 & 0.000019 \\
o3-mini                 & 0.169071 & 0.000030 & 0.147321 & 0.000027 & 0.106837 & 0.000025 & 0.116984 & 0.000025 \\
o4-mini                 & 0.196871 & 0.000037 & 0.191874 & 0.000027 & 0.169686 & 0.000035 & 0.170538 & 0.000033 \\
\midrule
gpt-oss-20B            & -0.326036 & 0.000023 & -0.336262 & 0.000016 & -0.384107 & 0.000007 & -0.374236 & 0.000009 \\
gpt-oss-120B           & -0.209097 & 0.000025 & -0.234670 & 0.000030 & -0.240233 & 0.000017 & -0.238353 & 0.000029 \\
gpt-5                  & 0.239273 & 0.000022 & 0.261642 & 0.000026 & 0.174972 & 0.000012 & 0.233048 & 0.000021 \\
\midrule
Claude Sonnet 4        & 0.092875 & 0.000033 & 0.067902 & 0.000030 & 0.067379 & 0.000028 & 0.064168 & 0.000021 \\
Claude Opus 4          & 0.305950 & 0.000053 & 0.289900 & 0.000054 & 0.232681 & 0.000024 & 0.229159 & 0.000035 \\
Claude Opus 4.1        & 0.480285 & 0.000133 & 0.483796 & 0.000129 & 0.439709 & 0.000078 & 0.431055 & 0.000079 \\
\midrule
Qwen-3-0.6B            & 0.270375 & 0.000001 & 0.213543 & 0.000004 & 0.170712 & 0.000001 & 0.208552 & 0.000005 \\
Qwen-3-1.7B            & 0.334988 & 0.000033 & 0.305905 & 0.000028 & 0.291700 & 0.000018 & 0.286917 & 0.000022 \\
Qwen-3-4B              & 0.230608 & 0.000039 & 0.247055 & 0.000031 & 0.189426 & 0.000022 & 0.182345 & 0.000025 \\
Qwen-3-8B              & 0.306297 & 0.000040 & 0.301868 & 0.000026 & 0.271688 & 0.000023 & 0.283367 & 0.000020 \\
Qwen-3-14B             & 0.289547 & 0.000064 & 0.269849 & 0.000048 & 0.236769 & 0.000039 & 0.277388 & 0.000033 \\
Qwen-3-32B             & 0.350025 & 0.000056 & 0.339766 & 0.000055 & 0.235912 & 0.000077 & 0.292756 & 0.000049 \\
Qwen3-30B-A3B          & 0.385370 & 0.000051 & 0.387557 & 0.000067 & 0.304281 & 0.000056 & 0.344193 & 0.000041 \\
Qwen3-235B-A22B        & 0.446157 & 0.000049 & 0.430845 & 0.000062 & 0.389649 & 0.000068 & 0.370387 & 0.000062 \\
\midrule
Deepseek-R1            & -0.028327 & 0.000003 & -0.012076 & 0.000010 & -0.035678 & 0.000007 & -0.051757 & 0.000002 \\
DeepSeek-V3.1          & 0.358224 & 0.000040 & 0.331861 & 0.000045 & 0.278489 & 0.000018 & 0.292173 & 0.000029 \\
DeepSeek-V3.1-Terminus & 0.329678 & 0.000021 & 0.310457 & 0.000029 & 0.260189 & 0.000025 & 0.297023 & 0.000020 \\
DeepSeek-V3.2          & 0.316112 & 0.000029 & 0.282013 & 0.000040 & 0.246096 & 0.000030 & 0.256818 & 0.000029 \\
\bottomrule
\end{tabular}
\vspace{-.5em}
\caption{Performance of mainstream models on code generation and agentic benchmarks. Each benchmark shows ARISE scores and corresponding Scaling Metrics (SM). DeepSeek-V3.1, V3.1-Terminus, and V3.2 refer to the DeepSeek series models.}
\label{tab:arise_scaling_code_agentic}
\vspace{-1.5em}
\end{table*}
\end{document}